%% file: 01_colt_main.tex
\documentclass[final,12pt]{colt2026} 

\title{Convergence of Continual Learning in Homogeneous Deep Networks}
\usepackage{times}

\coltauthor{%
 \Name{Matan Schliserman} \Email{schliserman@mail.tau.ac.il}\\
 \addr Blavatnik School of Computer Science and AI, Tel Aviv University
 \AND
 \Name{Gon Buzaglo} \Email{gon.buzaglo@princeton.edu}\\
 \addr Princeton University
 \AND
 \Name{Itay Evron} \Email{itay@evron.me}
 \AND
 \Name{Daniel Soudry} \Email{daniel.soudry@gmail.com}\\
 \addr Department of Electrical and Computing Engineering, Technion%
}

\input{00_custom}

\begin{document}

\maketitle

\begin{abstract} 
We characterize weakly regularized continual classification in homogeneous models as sequential projections onto task margin sets.
This result generalizes prior analyses restricted to either stationary (single-task) deep models or continual linear models. 
We show that global convergence generally fails, even for simple models linear in data but nonlinear in parameters. 
Nevertheless, by leveraging results from nonconvex projection theory, we identify regularity properties of homogeneous deep networks that guarantee local linear convergence under random and cyclic task sequences.
Finally, we extend our analysis to continual regression, unifying the framework for homogeneous models.
\end{abstract}


\input{02_intro_and_setting}

\input{04_convergence}

\input{05_regression}

\section{Discussion}

Our results frame continual learning in homogeneous deep networks as a geometric process governed by sequential projections onto task-induced margin sets, unifying several empirical and theoretical observations. Below, we discuss additional connections to prior work and highlight promising directions for further inquiry.


\paragraph{Geometry as the source of catastrophic forgetting.} 
Our results suggest that \emph{catastrophic} forgetting is a fundamental geometric consequence of nonconvexity. 
Unlike linear models with convex feasible sets \citep{evron2023classification}, homogeneous DNNs induce nonconvex sets where even exact sequential projections may fail to reach a point in the joint feasible set. 
Our lower bound in \cref{sec:squared} demonstrates that this occurs even in low-dimensional models under benign task orderings, proving that nonconvex geometry alone is sufficient to preclude global convergence.

\paragraph{The effect of depth.}
The convergence rates in \Cref{thm:local_convergence,thm:local_convergence_reg} are governed by the ratio $r/G$, where $r$ is the homogeneity degree and $G$ is the Lipschitz constant. 
While depth increases $r$ (see \cref{eq:fc_homogeneity}), the Lipschitz constant $G$ typically grows exponentially with the number of layers \citep[e.g.,][]{ szegedy2014intriguing,virmaux2018lipschitz}.
Consequently, increased depth tends to decrease the $r/G$ ratio, thereby \emph{degrading} the local convergence rate.
This provides a theoretical mechanism for the findings of \citet{guha2024diminishing}, who observe that deeper networks suffer from exacerbated forgetting \citep[see also][]{mirzadeh2021wide, ramasesh2020anatomy}. 
A compelling avenue for future work is to characterize how specific architectures influence the geometry of margin sets (\cref{def:set_class}) and their corresponding regularity moduli.

\paragraph{Characterizing convergence in deep continual models.} 
Empirical studies have shown that task repetition can mitigate forgetting even without explicit algorithmic intervention \citep{lesort2022scaling, hemati2025continual}.
While recent analytical work has explained this in linear models through cyclic or random task orderings \citep{evron2022catastrophic,evron2023classification,evron2025better,swartworth2023nearly,levinstein2025optimal,attia2025fast}, such analyses often rely on the NTK regime to approximate deep models.
Our framework more directly reflects the intrinsic nonconvex dynamics of deep architectures.
We hope the analytical foundation laid here facilitates a deeper understanding of failure modes, the benefits of pretraining, and the conditions for local stability in deep continual learning.







\acks{The research of MS was Funded by the European Research Council (ERC) under the European Union’s Horizon 2020 research and innovation program (grant agreement No.\ 101078075).
Views and opinions expressed are however those of the author(s) only and do not necessarily reflect those of the European Union or the European Research Council. Neither the European Union nor the granting authority can be held responsible for them.
This work received additional support from the Israel Science Foundation (ISF, grant numbers 2549/19 and 3174/23), a grant from the Tel Aviv University Center for AI and Data Science (TAD) and
from the Len Blavatnik and the Blavatnik Family foundation.

GB thanks Elad Hazan and Princeton University for the financial support.

The research of DS was Funded by the European Union (ERC, A-B-C-Deep, 101039436). Views and opinions expressed are however those of the author only and do not necessarily reflect those of the European Union or the European Research Council Executive Agency (ERCEA). Neither the European Union nor the granting authority can be held responsible for them. 
}

\newpage
\bibliography{99_bibliography}

\newpage

\input{90_appendices}

\end{document}

%% file: 00_custom.tex
\usepackage[utf8]{inputenc} 
\usepackage[T1]{fontenc}    
\usepackage{hyperref}       
\usepackage{url}            
\usepackage{booktabs}       
\usepackage{amsfonts}       
\usepackage{nicefrac}       
\usepackage{microtype} 
\usepackage{natbib}
\usepackage{xcolor}         
\usepackage{amsfonts}
\usepackage{mathtools}
\usepackage[capitalize,noabbrev]{cleveref}
\usepackage{dblfloatfix}    
\usepackage[normalem]{ulem} 

\usepackage{algorithm,algorithmic}

\usepackage{amsmath}
\usepackage{amssymb}
\usepackage{bm}

\usepackage{verbatim}
\usepackage{tabto}
\usepackage{hyperref}
\usepackage{enumitem}
\usepackage{xspace}
\usepackage{xcolor}
\usepackage[font=small]{caption}

\usepackage{booktabs}
\usepackage{array}
\usepackage{arydshln}

\makeatletter

\let\c@subfigure\undefined
\makeatother
\usepackage{subcaption} 
\usepackage{graphicx}

\usepackage{aliascnt}

\definecolor{itay}{RGB}{210,30,220}
\definecolor{gon}{RGB}{40,180,220}
\definecolor{daniel}{RGB}{30,200,30}
\definecolor{todo}{RGB}{50,50,200}
\definecolor{red2}{RGB}{220,25,25}

\newcommand{\deleted}[1]{TODO}%

\newcommand{\unnotice}[1]{TODO}

\let\emptyset\varnothing

\newcommand{\ie}{\textit{i.e., }}
\newcommand{\eg}{\textit{e.g., }}

\newcommand{\vect}[1]{\mathbf{#1}}

\newcommand{\doubleN}{\mathbb{N}}

\newcommand{\naturals}{\doubleN^{+}}
\DeclareMathOperator*{\argmin}{argmin}
\DeclareMathOperator*{\argmax}{argmax}  
\newcommand{\vv}{\vect{v}}
           
\newcommand{\vz}{\vect{z}}
\newcommand{\vx}{\vect{x}}

\newcommand{\0}{\mat{0}}

\newcommand{\mat}[1]{\mathbf{#1}}
\newcommand{\X}{\mat{X}}
\newcommand{\W}{\mat{W}}

\newcommand{\I}{\mat{I}}

\newcommand{\B}{\mat{B}}

\newcommand{\x}{\vect{x}}
\newcommand{\vp}{\vect{p}}
\newcommand{\w}{\vect{w}}

\newcommand{\vtheta}{\bm{\theta}}
\newcommand{\vy}{\vect{y}}
\newcommand{\vu}{\vect{u}}
\newcommand{\y}{\vect{y}}
\newcommand{\vd}{\vect{d}}

\def\reals{\mathbb{R}}
\def\R{\mathbb{R}}

\newcommand{\tparagraph}[1]{\vspace{-5pt}\paragraph{#1}}

\newcommand{\hfrac}[2]{{{#1}/{#2}}}
\newcommand{\norm}[1]{\left\Vert{#1}\right\Vert}

\newcommand{\cnt}[1]{\left[{#1}\right]}

\newcommand{\expectation}{\mathop{\mathbb{E}}}

\newcommand{\iprod}[2]{\left\langle{#1},{#2}\right\rangle}
\newcommand{\prn}[1]{\left({#1}\right)}

\newcommand{\Bigprn}[1]{\Big({#1}\Big)}

\newcommand{\tprn}[1]{({#1})}

\newcommand{\smallnorm}[1]{\Vert{#1}\Vert}
\newcommand{\tnorm}[1]{\smallnorm{#1}}
\newcommand{\bignorm}[1]{\big\Vert{#1}\big\Vert}

\newcommand{\dist}{{d}}

\newcommand{\proj}{\bm{\Pi}}

\newcommand{\homogendeg}{r}
\newcommand{\paramdim}{p}
\newcommand{\ordcyc}{\tau_{\text{cyc}}}
\newcommand{\ordrnd}{\tau_{\text{iid}}}

\newcommand{\N}{\mathbb{N}}

\renewcommand{\epsilon}{\varepsilon}

\newcommand{\vecind}[2]{{#1}[{#2}]}

\newcommand{\opt}{\Theta^\star}

\usepackage{bm}

\newcommand{\ball}{\mathcal{B}}

\crefname{equation}{Eq.}{Eq.}

\newaliascnt{lemma}{theorem}
\newtheorem{lemma}[lemma]{Lemma}
\aliascntresetthe{lemma}

\newaliascnt{corollary}{theorem}
\newtheorem{corollary}[corollary]{Corollary}
\aliascntresetthe{corollary}

\newaliascnt{definition}{theorem}
\newtheorem{definition}[definition]{Definition}
\aliascntresetthe{definition}

\newaliascnt{proposition}{theorem}
\newtheorem{proposition}[proposition]{Proposition}
\aliascntresetthe{proposition}

\newaliascnt{claim}{theorem}
\newtheorem{claim}[theorem]{Claim}
\aliascntresetthe{claim}

\newcommand{\fset}{\mathcal{C}}

\newcommand{\interset}{\fset^\star}
\newcommand{\riterate}[1]{\Theta_{#1}^{(\lambda)}}
\newcommand{\rjiterate}[1]{\Theta_{#1}^{(\lambda_j)}}
\newcommand{\piterate}[1]{\bar{\Theta}_{#1}}

\RequirePackage{thmtools}

\RequirePackage{crossreftools}

\newtheorem{assumption}{Assumption}

\crefname{claim}{Claim}{Claims}
\crefname{assumption}{Assumption}{Assumptions}

\usepackage{titletoc} 
\newcommand{\appendixtableofcontents}{%
  \printcontents[app]{l}{1}{\section*{Appendix Contents}}
}

\makeatletter
\renewenvironment{proof}[1][]{%
  \par
  \normalfont \topsep6\p@\@plus6\p@\relax
  \trivlist
  \item[\hskip\labelsep\bfseries
    \proofname
    \if\relax\detokenize{#1}\relax\else\ #1\fi 
    \@addpunct{.}]\ignorespaces
}{%
  \jmlrQED\endtrivlist
}
\makeatother



%% file: 02_intro_and_setting.tex
\section{Introduction}

Continual learning focuses on training models on a sequence of tasks to incrementally accumulate expertise. 
This paradigm has gained traction with large-scale foundation models, where full retraining is often prohibitive due to computational costs, privacy regulations, or data retention constraints. 
In these settings, models must be finetuned sequentially as data arrives, while avoiding catastrophic forgetting \citep{robins1995rehearsal,van2025continual} and maintaining plasticity for future learning \citep{dohare2024plasticity}. 
Alongside the development of practical algorithms (\eg \citealt{qiao2024prompt,behrouz2025nested}; see \citealt{yang2025recent} for a survey), 
deepening our theoretical understanding of the underlying dynamics remains essential. 
We address this by developing analytical tools specifically for continual learning in homogeneous deep neural networks (DNNs).

Theoretical literature has largely centered on continual linear regression due to its tractability, shedding light on critical aspects of the field: 
convergence under various task orderings \citep{evron2022catastrophic,evron2025better,swartworth2023nearly,attia2025fast}, 
task similarity \citep[e.g.,][]{Asanuma_2021,lin2023theory,li2025optimal,tsipory2025greedy}, 
regularization \citep[e.g.,][]{zhao2024statistical,levinstein2025optimal,karpel2026regularization}, 
overparameterization \citep{goldfarb2023analysis,goldfarb2024theJointEffect}, 
and algorithmic effects \citep{doan2021NTKoverlap,peng2022featureExtract,peng2023ideal}.
Other notable works study continual linear classification, primarily deriving convergence guarantees under explicit or implicit regularization
\citep{evron2023classification,jung2025convergence}.
While illuminating, linear analysis cannot fully account for the complex, nonlinear dynamics of modern DNNs.


\pagebreak

Some recent works analyze continual learning in \emph{nonlinear} models, largely focusing on simplified settings, such as two-layer models that enable precise analyses \citep{lee2021taskSimilarity,lee2022maslow,li2025twolayer,taheri2025theory}. 
Others leverage perturbative, mean-field, or scaling-limit techniques, including regimes that interpolate between lazy and feature-learning dynamics, but are typically restricted to shallow architectures or infinite-width limits \citep{shan2024order,graldi2025lazy}.
As a result, depth influences forgetting primarily through static feature statistics or last-layer adaptation. In contrast, our work analyzes continual learning in homogeneous deep networks of arbitrary depth, with all layers trained and without relying on lazy-training or infinite-width assumptions, thereby capturing fully nonlinear feature evolution across layers.


A popular approach to stabilizing continual models and preventing forgetting is the use of explicit regularization in the parameter space \citep[e.g.,][]{kirkpatrick2017ewc}.
These methods aim to bias the optimization toward previous iterates via:
\begin{align}
\label{eq:regularized_cl}
\riterate{t}
\leftarrow
{\argmin}_{\Theta}
\big\{
\mathcal{L}_{t} (\Theta)
+
\lambda \bignorm{\Theta - \riterate{t-1}}^2_{\B_t}
\big\}
\,,
\end{align}
where $\B_{t} \succeq 0$ skews the regularization toward directions deemed ``important'' for previous tasks---\eg based on their Fisher information \citep[see][]{benzing2021unifying_regularization}. 
Interestingly, even isotropic regularization ($\B_{t}=\I$) has proven beneficial both practically \citep[e.g.,][]{lubana2021regularization,smith2022closer} and theoretically \citep{li2023fixed,levinstein2025optimal,karpel2026regularization}. 
Our work adopts a standard analytical framework with \emph{weak} isotropic regularization in the limit as $\lambda \downarrow 0$, as detailed next.

The limit of \emph{weak} regularization is an influential regime for analyzing algorithmic biases---extending beyond continual learning---due to its tractability and connections to minimum-norm and max-margin solutions \citep[e.g.,][]{hastie2022surprises, aubin2020generalization}.
Specifically, in ``traditional'' stationary classification, \citet{rosset2004margin} proved that weakly regularized \emph{linear} models converge\footnote{
In this introduction, we say that 
$\vtheta^{(\lambda)}$  converges in direction to 
$\bar\vtheta$
if
$\lim_{\lambda\downarrow 0}
\frac{\vtheta^{(\lambda)}}{\tnorm{\vtheta^{(\lambda)}}}
=\frac{\bar\vtheta}{\tnorm{\bar\vtheta}}
$.
} to their max-margin counterparts. 
That is, linear models trained with margin-based losses---\eg the logistic loss---on separable data recover the Hard-Margin SVM solution as regularization vanishes:
\begin{align}
\vtheta^{(\lambda)}
=
\argmin_{\vtheta}
\big\{
\mathcal{L}(\vtheta)
+
\lambda \norm{\vtheta}^2
\big\}
~
\xrightarrow[\quad\lambda\downarrow 0\quad]{\text{\,in direction\,}}
~
\argmax_{\tnorm{\bar\vtheta}\le 1}
\min_{i}
y_i {\bar\vtheta}^{\top}\x_i
=
\!\!\!
\argmin_{y_i {\bar\vtheta}^{\top}\x_i \ge 1,\; \forall i}
\!\!\!\!
{\tnorm{\bar\vtheta}}
\,
.
\end{align}
A later work by \citet{Wei2019} generalized this to DNNs, showing that 
the normalized margin of a weakly-regularized homogeneous model $f$ converges to the max margin:
\begin{align}
\min_{i} y_{i} 
f(\x_{i};\tfrac{\riterate{}}{\tnorm{\riterate{}}})
~
\xrightarrow[\quad\lambda\downarrow 0\quad]{}
~
\max_{\tnorm{\Theta}\le 1}
\min_{i} y_{i} 
f(\x_{i};\Theta)
\,.
\end{align}

In continual learning, \citet{evron2023classification} used weak regularization to study jointly-separable linear classification. 
They showed an equivalence between the iterates of weakly regularized continual linear classification and those of a sequential margin-separating projection algorithm.
That is, defining a convex ``margin set''
$\fset_{t}
\triangleq
\big\{
\bar\vtheta
~\big|~
y_i^{(t)} {{\bar\vtheta}}^{\top}\x_i^{(t)}
\ge 1,
\ \forall i\}$, they showed:
\begin{align}
\Bigprn{
\vtheta_{t}^{(\lambda)}
=
\argmin_{\vtheta}
\big\{
\mathcal{L}_{t}(\vtheta)
+
\lambda \tnorm{\vtheta-\vtheta_{t-1}^{(\lambda)}}^2
\big\}
}_{t}
~
\xrightarrow[\quad\lambda\downarrow 0\quad]{\text{\,in direction\,}}
~
\Bigprn{
\bar{\vtheta}_{t}
=
\argmin_{\bar{\vtheta}\in \fset_{t}}
\bignorm{\bar{\vtheta}-\bar{\vtheta}_{t-1}}^2
}_{t}
\,.
\end{align}

We complete the analysis landscape 
by proving that weakly-regularized continual learning with homogeneous DNNs projects onto \emph{nonconvex} margin sets, 
$\fset_{t}
\triangleq
\big\{
\bar\Theta
~\big|~
y_i^{(t)} 
f\tprn{\x_i^{(t)}; \bar\Theta}
\ge 1,
\ \forall i\}$:
\begin{align}
\Bigprn{
\riterate{t}
=
\argmin_{\Theta}
\big\{
\mathcal{L}_{t}(\Theta)
+
\lambda \tnorm{\Theta-\riterate{t-1}}^2
\big\}
}_{t}
\,
\xrightarrow[\quad\lambda\downarrow 0\quad]{\text{\,in direction\,}}
\,
\Bigprn{
\piterate{t}
\in
\argmin_{\piterate{}\in \fset_{t}}
\bignorm{\piterate{}-\piterate{t-1}}^2
}_{t}
\,.
\end{align}

\begin{table}[t]
\centering
\small 
\vspace{-1.05em}
\caption{Landscape of weak-regularization analysis across learning regimes and model classes.}
\vspace{-.25em}
\label{tab:comparison}
\begin{tabular}{lcc}
\toprule
& \textbf{Linear models} & \textbf{Homogeneous DNNs} \\
\midrule
\textbf{Stationary setting} 
& \citet{rosset2004margin}
& \citet{Wei2019} \\
\textbf{Continual learning} 
& \citet{evron2023classification} 
& {Our work} (2026) \\
\bottomrule
\end{tabular}
\vspace{-0.8em}
\end{table}


After establishing an equivalence between continual classification in homogeneous models and sequential projections, we first show that, unlike previous results for linear models, global convergence is generally not guaranteed in homogeneous models.
Encouragingly, we prove \emph{local} convergence guarantees for these models through tools from the literature on nonconvex projections \citep[e.g.,][]{lewis2008alternating,dao2019linear}. Lastly, we demonstrate the generality of these tools by extending the analysis from classification to regression. 
Overall, these tools provide a principled framework for understanding continual learning in deep neural networks. 

\paragraph{Summary of contributions.}
The main contributions of this paper are:
\begin{enumerate}[leftmargin=0.45cm, itemindent=0cm, itemsep=0pt,labelsep=0.2cm,topsep=3pt]

\item We prove that weakly-regularized continual learning with homogeneous DNNs implicitly performs sequential margin projections.

\item We show that, unlike in prior work on linear models, these projection sets are not necessarily convex, leading to qualitative differences between settings.

\item We demonstrate that global convergence is not guaranteed; consequently, forgetting can be \emph{catastrophic}, even for a simple homogeneous model with only $4$ parameters and $2$ tasks.

\item Bridging to the nonconvex projection literature, we establish local convergence guarantees for random and cyclic task orderings, with rates depending on model Lipschitzness and smoothness.

\item Finally, we extend our results to continual regression in homogeneous models, providing a unified framework for both classification and regression. 

\end{enumerate}

\section{Continual Classification in Homogeneous DNNs as Sequential Projections}
\label{sec:setup}

\paragraph{Notation.}
We denote $\cnt{n}\triangleq\{1,2,\dots,n\}$,
$[z]_{+}\triangleq \max\{0, z\}$,
and $
\ball_{\delta}(\vu)
\triangleq \{\vv~|~ \|\vv-\vu\|\le \delta\}
$.
We call a function $f:\reals^{\paramdim}\to\reals$ \textit{positively homogeneous} of degree $\homogendeg\in \reals$ if
$f(c \vz)=c^{\homogendeg}f(\vz),
\forall c > 0$.

\paragraph{Models.}
Throughout the paper we consider $\homogendeg$-positively-homogeneous models parametrized by weights \(\Theta\in\reals^\paramdim\).
For example, a fully connected neural network with $L\ge 1$ layers, no bias terms, and a positively homogeneous activation function of degree $1$ (such as ReLU or leaky ReLU) is $L$-positively-homogeneous in its parameters 
$\Theta=\prn{\W_1, \dots, \W_{L}}$:
\begin{align}
\label{eq:fc_homogeneity}
f(\x;\Theta)
=
\W_{L} \sigma \!\prn{
\W_{L-1} \sigma\!\prn{
\cdots
\sigma \prn{
\W_{1} \x
}
}
}
\quad
\Longrightarrow
\quad
f(\x;c\Theta)
=
c^{L}f(\cdot;\Theta),
\quad
\forall c>0
\,.
\end{align}

In this section, we focus on binary classification to establish our fundamental result: weakly regularized homogeneous models in continual learning converge to a sequential projection algorithm. We extend these results to regression in \cref{sec:regression}. In both settings, the projection perspective bridges our analysis to existing literature on nonconvex projections, facilitating rigorous convergence results.

\subsection{Setting: Continual Classification}
\label{sec:setting_classification}

\paragraph{Learning setup.}
We consider a setup with $M$ classification tasks 
$\prn{\X^{(1)},\y^{(1)}},\dots, \prn{\X^{(M)},\y^{(M)}}$
consisting of feature matrices $\X^{(m)}\in\reals^{n_{m}\times d}$
and binary labels $\y^{(m)}\in\{-1,1\}^{n_m}$.
The learner is exposed to tasks sequentially in $k\in\naturals$ iterations according to a \emph{task ordering} $\tau:\cnt{k}\to\cnt{M}$, which are commonly assumed to be cyclic or random, to capture adversarial vs.\ nonadversarial behaviors \citep{evron2022catastrophic,evron2023classification,swartworth2023nearly,levinstein2025optimal}. 
That is, for iteration $t\in \cnt{k}$,
\begin{align}
\ordcyc(t) = (\prn{t-1} \bmod M) + 1,
\qquad
\ordrnd(t) \stackrel{\text{i.i.d.}}{\sim} \mathrm{Unif}\{1,\dots,M\}
\,.
\end{align}


Our analysis assumes the model can separate all $M$ training sets. 
This assumption follows prior theoretical work on continual linear classification \citep{evron2023classification,jung2025convergence}, where it is justified by standard teacher assumptions or high-dimensional randomness \citep[e.g.,][]{cover1965geometrical}. 
For the deeper architectures considered here, joint separability is a significantly milder requirement; sufficiently expressive networks can interpolate finite datasets under general conditions, a phenomenon rooted in universal approximation results \citep{cybenko1989approximation}.

\begin{assumption}[Joint Separability]
\label{asm:overparam-separability}
Define the individual feasible sets
\begin{equation}
    \label{def:set_class}
\fset_{m}\triangleq
\Big\{
\bar\Theta
~\Big|~
y_{i}^{(m)}
f(\x_i^{(m)};\bar\Theta)
\ge 1,
\ \forall i \in \cnt{n_{m}}
\Big\},\qquad
\forall m \in \cnt{M}
\,.
\end{equation}
We assume the intersection, \ie the joint feasible set, is nonempty:
$\interset \triangleq \fset_{1}\cap\dots\cap\fset_{M}
\neq \emptyset$.
\end{assumption}

\paragraph{Metrics.}
We follow the projection analysis for continual classification in \citet{evron2023classification} and derive results for three quantities of interest:
\begin{enumerate}[leftmargin=0.5cm, itemsep=3pt,itemindent=0cm,labelsep=0.2cm]
\item \textbf{Distance to joint feasible set:}
\hfill
$\dist(\bar\Theta, \interset)
=
\min_{\bar\Theta^{\star} \in \interset}
\tnorm{\bar\Theta - \bar\Theta^{\star}}$.

\item \textbf{Distance to individual feasible set:}
\hfill
$\forall m\in\cnt{M},~
\dist(\bar\Theta, \fset_{m})
=
\min_{\bar\Theta{'} \in \fset_{m}}
\tnorm{\bar\Theta - \bar\Theta{'}}$.

\item \textbf{Forgetting:}
At iteration $t$, we define the forgetting of a task previously seen at iteration $t'\le t$ as the maximal hinge loss over its training points:
\hfill
$
F_{\tau(t')} (\bar\Theta_{t}) = 
\max_{i}
\big[
1-y_{i} f(\x_i^{\tau(t')};\bar\Theta_{t})
\big]_{+}.
$
\end{enumerate}
As noted by \citet{evron2023classification}, the hinge loss possesses favorable properties for analyzing continual learning. 
First, we will show that immediately after learning task $m$, the iterate resides in the feasible set, $\bar{\Theta} \in \fset_m$; consequently, the forgetting of the most recent task is zero. 
More importantly, lower forgetting---quantified here by a lower hinge loss---implies improved training margins, which are often tied to better generalization.


\subsection{Fundamental Result: Continual Classification to Sequential Margin Projections}
\label{sec:main_thm}
\paragraph{Learning algorithm.}
We consider a learner that minimizes the margin-based logistic loss,\footnote{
While we measure forgetting using the hinge loss---which can reach zero once the margin is sufficiently large---the learner optimizes the logistic loss. 
This choice aligns our model with standard deep learning practices and is made for simplicity (\ie our analysis can be extended to other common loss functions).}
defined for each task $m\in\cnt{M}$ as
\begin{equation}
\mathcal L_m(\Theta)
\triangleq
\frac{1}{n_m}
\sum_{i=1}^{n_m}\log\prn{1+\exp(-\,y_i^{(m)} f(\x_i^{(m)};\Theta))}.
\end{equation}
In \cref{alg:regularized_cl} we adopt a common strategy in continual learning by regularizing parameter shifts to prevent catastrophic forgetting. 
Such methods effectively trade off model plasticity for increased stability. 
While several prominent works utilize weighted $L_2$ regularization
(often guided by Fisher information) to protect task-specific parameters \citep{kirkpatrick2017ewc,zenke2017continual,aljundi2018memory,benzing2021unifying_regularization}, we focus on isotropic $L_2$ regularization. 
This variant offers a lower memory footprint and greater analytical tractability while remaining empirically competitive \citep{Hsu2018Baselines,lubana2021regularization,smith2022closer}. Furthermore, isotropic regularization has become a focal point of recent theoretical inquiries into model consolidation \citep[e.g.,][]{li2023fixed,levinstein2025optimal,karpel2026regularization}.


Formally, in \cref{alg:regularized_cl} we study the regularization at the limit as $\lambda\downarrow0$ and show that its iterates converge to the sequential projections of \cref{alg:sequential_mm}.

\vspace{-1.1em}
\begin{figure}[h]
    \centering
    \begin{minipage}{0.5\textwidth}
        \begin{algorithm}[H]
          \caption{Regularized Continual Learning
          \label{alg:regularized_cl}}
        \begin{algorithmic}[0]
           \STATE {
           \hspace{-1em}
           \textbf{Initialization:} 
           $\riterate{0}=\0$
           } 
           \STATE { 
           \hspace{-1em}
           \textbf{For each} iteration $t=1,\dots,k$: 
           }
           \STATE {
            \vspace{2pt}
                \hspace{-0.3em}%
                $\displaystyle
                  \riterate{t} \gets
                  \argmin_{\Theta}
                  \big\{
                    \mathcal{L}_{\tau(t)} (\Theta)
                    +
                    \lambda \bignorm{\Theta - \riterate{t-1}}^2
                  \big\}
                $
              }%
            \vspace{3pt}
           \STATE {
            \parbox[c][1.7em][c]{\linewidth}{%
           \hspace{-1em}
           \textbf{Output:} $\riterate{k}$
           }
           } 
        \end{algorithmic}
        \end{algorithm}
    \end{minipage}
    \hfill
    \begin{minipage}{0.49\textwidth}
        \begin{algorithm}[H]
          \caption{Sequential Margin Projections
          \label{alg:sequential_mm}}
        \begin{algorithmic}
        \vspace{0.25em}
           \STATE {
           \hspace{-1em}
           \textbf{Initialization:}
           $\piterate{0}=\0$
           } 
           \STATE { 
           \hspace{-1em}
           \textbf{For each} iteration $t=1,\dots,k$: 
           }
           \STATE {
            \vspace{2pt}
              \vspace{3pt}
               \hspace{-.3em}
               $\displaystyle
                    \piterate{t} 
                    \gets
                    \proj_{\tau(t)}\tprn{\piterate{t-1}}
                    \triangleq
                    \argmin_{\bar{\Theta}
                    \in\mathcal \fset_{\tau(t)}}
                    \bignorm{\bar{\Theta}-\piterate{t-1}}^2
                    $
                }
            \vspace{-2pt}
           \STATE {
            \parbox[c][1.6em][c]{\linewidth}{%
           \hspace{-1em}
           \textbf{Output:} $\piterate{k}$
           }
           } 
        \end{algorithmic}
        \end{algorithm}
    \end{minipage}
\end{figure}
\vspace{-1em}

\begin{theorem}[Weakly-Regularized CL $\to$ Sequential Margin Projections]
\label{thm:gen_m_convergence}
Consider an \linebreak $r$-positively-homogeneous model $f(\cdot;\Theta) : \mathcal{X} \to \mathbb{R}$, \ie $f(x; c\Theta) = c^r f(x;\Theta)$ for all $c > 0$. 
Assume individual separability, \ie nonempty feasible sets $\fset_{1}, \dots, \fset_{M}$ (implied by \Cref{asm:overparam-separability}). 
Then, as $\lambda\downarrow 0 $, \cref{alg:regularized_cl} trained with the logistic loss aligns with \Cref{alg:sequential_mm}.
\linebreak 
That is, for every iteration $t \in [k]$, any sequence $\lambda \downarrow 0$ admits a subsequence $(\lambda_j)$ and a point $\Bar{\Theta}_t\in \proj_{\tau(t)}\tprn{\piterate{t-1}}\triangleq
    \argmin_{\bar{\Theta}
    \in\mathcal \fset_{\tau(t)}}
    \bignorm{\bar{\Theta}-\piterate{t-1}}^2$ such that
\[
    \frac{\Theta_t^{(\lambda_j)}}{\bignorm{\Theta_t^{(\lambda_j)}}} 
    \xrightarrow{\quad j \to \infty \quad} 
    \frac{\bar{\Theta}_t}{\bignorm{
        \bar{\Theta}_t
        \vphantom{\Theta_t^{(\lambda_j)}} 
    }}.
\]
\end{theorem}

The full proof is provided in \cref{app:proof_main}. 
Below, we describe our proof techniques, contrast them with prior work, and outline the core ideas behind the derivation.

\paragraph{Comparison to prior work.} 
As discussed in the introduction, the limit of weak regularization is a standard analytical tool. 
In stationary classification---equivalent to the \emph{first} iteration of \cref{alg:regularized_cl,alg:sequential_mm}---weakly-regularized solutions have been linked to max-margin solutions in both linear \citep{rosset2004margin} and homogeneous models \citep{Wei2019}. 
More recently, \citet{evron2023classification} extended this framework to \emph{continual} linear classification.
Similar to our approach, they adjust the projection reference point to capture how expertise accumulates across tasks. 
Our work completes this landscape by analyzing weakly regularized continual learning in the broader class of nonlinear, homogeneous DNNs (see \cref{tab:comparison}).

Technically, our proof of \cref{thm:gen_m_convergence} differs from previous approaches. 
For instance, while \citet{evron2023classification} rely on convex geometry and KKT optimality, these tools do not readily extend to our setting because the feasible sets induced by DNNs are nonconvex. 
Furthermore, \citet{Wei2019} rely on the pointwise convergence of the loss to a fixed limit. This notion of convergence is insufficient for the continual setting, where the regularization term is dynamic and defined relative to the shifting limit of the previous task.


To address these challenges, we employ the framework of $\Gamma$-convergence. 
This allows us to handle nonconvex geometry by analyzing the global minimizers of a sequence of evolving functionals. 


%
\begin{definition}[$\Gamma$-convergence; {\citealp[Theorem 2.1]{braides2006handbook}}]
Let $G_n:
\R^\paramdim \to(-\infty,+\infty]$ and \linebreak
$G:\R^\paramdim\to(-\infty,+\infty]$.
We say that $G_n$ $\Gamma$-converges to $G$ if:
\begin{enumerate}[label=(\alph*),leftmargin=*,itemsep=0pt]
\item For every $\Theta\in \R^\paramdim $ and every sequence $\Theta_n\to \Theta$,
$
G(\Theta)\le \liminf_{n\to\infty} G_n(\Theta_n).
$
\item For every $\Theta\in \R^\paramdim $ there exists a sequence $\Theta_n\to \Theta$ such that
$
G(\Theta)\ge \limsup_{n\to\infty} G_n(\Theta_n).
$
\end{enumerate}
\end{definition}


The key property of $\Gamma$-convergence is that the $\Gamma$-limit of a sequence of functionals ensures the convergence of their respective minimizers to a minimizer of the limit, as formalized next.

\begin{lemma}[Fundamental Theorem of $\Gamma$-convergence; {\citealp[Theorem 2.10]{braides2006handbook}}]
\label{thm:fundamental_gamma}
Let
\linebreak
$(G_n)_{n\in \N}$ be a family of functionals $G_n : \R^\paramdim \to (-\infty,+\infty]$. 
Assume that $(G_n)_{n\in \N}$ is equicoercive, \ie that
for every $C\in\mathbb R$, 
$\{\Theta\in \R^\paramdim : G_n(\Theta)\le C\} \subset K_C$ for some compact set $K_C$.
Then, if $G_n$ $\Gamma$-converges to a functional $G:\R^p \to(-\infty,+\infty]$ as $n\to\infty$,
the following hold:
\begin{enumerate}[label=(\alph*),leftmargin=*]
\item \textbf{Convergence of minimal values:}
$
\lim_{n\to\infty}\ \inf_{\Theta\in \R^\paramdim} G_\lambda(\Theta)
=
\min_{\Theta\in \R^\paramdim} G(\Theta).
$
\item \textbf{Convergence of minimizers:}
Every sequence $(\Theta_n)$ such that $\Theta_n \in \argmin G_n$ admits a convergent subsequence $(\Theta_{n_{j}})$ such that
$
\lim_{j\to\infty}\Theta_{n_{j}} =\opt$,
for $\opt \in \argmin G$.
If the minimizer $\opt$ of $G$ is unique, then the entire sequence $(\Theta_n)$ converges to $\opt$.
\end{enumerate}
\end{lemma}

Having established these preliminaries, we now detail the key ideas of our proof.

\paragraph{Proof sketch of our \cref{thm:gen_m_convergence}.}
Let $(\lambda_n)$ satisfy $\lambda_n \to 0$. For $\lambda>0$ and $t\in[k]$, define
$c_{\lambda,\homogendeg} \!=\! (\log \tfrac{1}{\lambda})^{\hfrac{1}{\homogendeg}}$ and
$\hat{\Theta}_t^{(\lambda)} \!=\! \Theta_t^{(\lambda)}/c_{\lambda,\homogendeg}$.
To establish convergence in direction, we first show convergence of the scaled iterates $\hat{\Theta}_t^{(\lambda)}$. 
For this, we need to prove that there exists a subsequence $(\lambda_\ell)$ such that
$\hat{\Theta}_t^{(\lambda_\ell)} \to \bar{\Theta}_t$ for some
$\bar{\Theta}_t \in 
\argmin_{\bar\Theta \in \fset_{\tau(t)}}
\|\bar\Theta - \bar{\Theta}_{t-1}\|^2$.

We proceed by induction on $t$.
Assume that for some subsequence $(\lambda_j)$,
$\hat{\Theta}_{t-1}^{(\lambda_j)} \to \bar{\Theta}_{t-1}
\in \argmin_{\bar\Theta \in \fset_{\tau(t-1)}}
\|\bar\Theta - \bar{\Theta}_{t-2}\|^2$,
we define
\[
G_t^{(\lambda)}(\Theta)
\triangleq
\frac{\mathcal L_{\tau(t)}(c_{\lambda,\homogendeg}\Theta)}
{c_{\lambda,\homogendeg}^2\lambda}
+
\|\Theta - \hat\Theta_{t-1}^{(\lambda)}\|^2 .
\]
A change of variables yields
\[
\Theta_t^{(\lambda)} \in \argmin
\{\mathcal L_{\tau(t)}(\Theta)
+ \lambda\|\Theta-\Theta_{t-1}^{(\lambda)}\|^2\}
\quad\iff\quad
\hat{\Theta}_t^{(\lambda)} \in \argmin G_t^{(\lambda)}(\Theta).
\]
Then, we show that,
along $(\lambda_j)$, the functionals $G_t^{(\lambda_j)}$ are equicoercive and
$\Gamma$-converge to\linebreak
$G_t(\Theta) = \mathbf 1_{\fset_{\tau(t)}}(\Theta)
+ \|\Theta - \bar{\Theta}_{t-1}\|^2$, where $\mathbf 1_{\mathcal{C}}(\Theta)=0$ for $\Theta\in\mathcal{C}$ and $\infty$ otherwise.
By \cref{thm:fundamental_gamma}, any sequence of minimizers
$\hat{\Theta}_t^{(\lambda_j)}$ admits a convergent subsequence whose limit is a
global minimizer of $G_t$, establishing the inductive step. 
Since this holds for
any $(\lambda_n)$ with $\lambda_n \to 0$, and since the directions of the scaled iterates $\hat{\Theta}_t^{(\lambda)}$ and the original iterates ${\Theta}_t^{(\lambda)}$ are identical, the theorem follows.




%% file: 04_convergence.tex
\section{Convergence Analysis from a Projection Perspective}
\label{sec:conv}
The established projection perspective (\cref{thm:gen_m_convergence}) facilitates analyzing continual learning through the lens of projection theory. 
Indeed, prior work has used classical results from alternating projections and Projections Onto Convex Sets (POCS) to analyze the convergence of continual \emph{linear} models \citep{evron2022catastrophic, evron2023classification}. 
However, these tools are not directly applicable in our nonlinear setting as feasible sets are no longer convex.
While this raises several challenges—detailed below—the projection viewpoint remains informative for structuring our analysis of the induced dynamics.


\subsection{Challenges: Projections onto Nonconvex Sets}
Extending projection-based analysis beyond convexity breaks classical analytical frameworks. 

\begin{figure}[t!]
  \captionsetup[subfloat]{width=0.31\textwidth}
  \setlength{\fboxsep}{0pt} 
  \setlength{\fboxrule}{.5pt} 
  \centering
  \subfloat[\textbf{Linear model.}
  Consider two 2D data points 
  $\x_1 = \left\lbrack5, -1\right\rbrack^{\top},\,
  \x_2 = \left\lbrack2,2\right\rbrack^{\top}$, $y_1=y_2=+1$,
  and a linear model 
  $f(\x;w_1,w_2)
  =
  w_1 x_1 + w_2 x_2
  $.
  As explained in \citet{evron2023classification}, the resulting feasible set is a closed and convex affine polyhedral cone of the form
  $
  \left\lbrack\begin{smallmatrix} 
  5 & -1 \\ 2 & 2
  \end{smallmatrix}\right\rbrack
  \left\lbrack\begin{smallmatrix} 
  w_{1} \\ w_{2}
  \end{smallmatrix}\right\rbrack
  \ge
  \left\lbrack\begin{smallmatrix} 
  1 \\ 1
  \end{smallmatrix}\right\rbrack$.
  ]{
    \label{fig:linear_fset}
    {\includegraphics[width=0.3\textwidth]{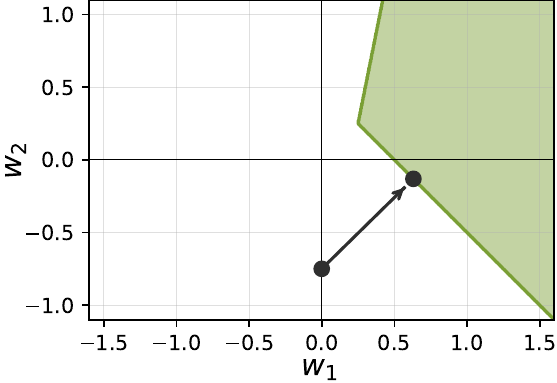}}
  }
  \hfill
  \subfloat[
  \textbf{Simplified ReLU.}
  Consider a single 1D data point $x=2,y=+1$, and a 1-positively-homogeneous model
  $f(x;w_1,w_2)
  =
  \lbrack w_1 x\rbrack_{+} + \lbrack w_2 x\rbrack_{+}
  $.
  \linebreak
  Here, the induced feasible set $f(2;w_1,w_2)\ge 1$ is nonconvex.
  Consequently, a single point can have multiple valid projections.
  ]{
    \label{fig:relu_fset}
    {\includegraphics[width=0.3\textwidth]{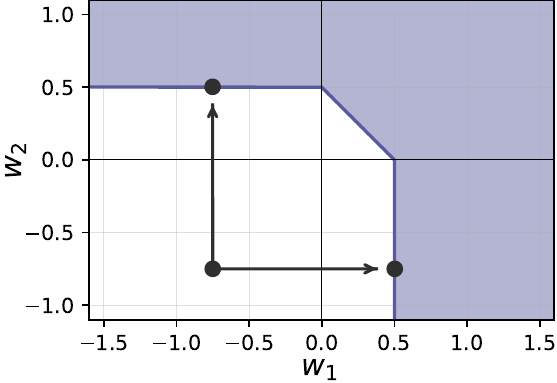}}
  }
  \hfill
  \subfloat[
  \textbf{Squared linear model.}
  Consider a single 1D data point $x=2,y=+1$,\linebreak and a 2-positively-homogeneous model
  $f(x; w_1,w_2)
  =
  w_1^2 x - w_2^2 x
  $ as in \citet{woodworth2020kernel}.
  While comprised of two convex regions, the resulting feasible set ${f(2;w_1,w_2)\ge 1}$ is nonconvex.
  ]{
    \label{fig:squared_fset}
    {\includegraphics[width=0.3\textwidth]{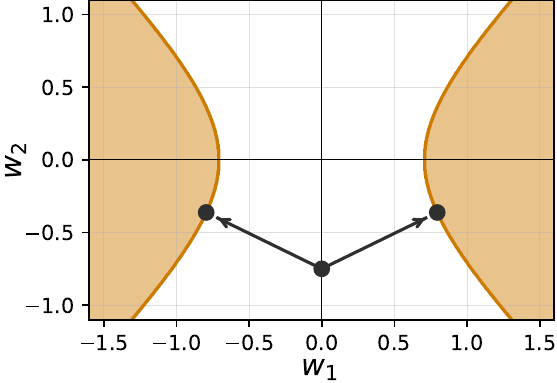}}
  }
  \caption{
  \textbf{Feasible sets under homogeneous models are not necessarily convex.}
  In the two-parameter spaces depicted, only the linear model yields convex feasible sets and, consequently, unique projections.}
  \label{fig:feasible_sets}
  \vspace{-.8em}
\end{figure}

\tparagraph{Nonunique projections.}
As illustrated in \Cref{fig:feasible_sets}, linear models induce convex feasible sets for which projections are uniquely defined---but general homogeneous models yield nonconvex feasible sets where a projection $\proj_{\fset}(\vu) \in \argmin_{\vv \in \fset} \|\vu-\vv\|_2$ may admit multiple minimizers.
This nonuniqueness is critical: \Cref{alg:sequential_mm} may select an arbitrary minimizer at each step, potentially inducing qualitatively distinct dynamics and branching into trajectories with diverging behaviors. 
Such branching complicates any uniform analysis that must account for all admissible selection rules.

\tparagraph{Breakdown of projection properties.}
Another distinction is that convex projections monotonically approach the \emph{joint} feasible set,
\ie $\dist(\proj_{\fset_m}(\bar\vtheta), \interset) \le \dist(\bar\vtheta, \interset)$, due to operator nonexpansiveness, 
\ie $\|\proj_{\fset}(\vu) - \proj_{\fset}(\vv)\| \le \|\vu - \vv\|$ 
\citep[][Lemma~4.5 and Corollary~D.1]{evron2023classification}. 
While crucial for convergence, these properties fail in the nonconvex regime; for instance, nonuniqueness in \cref{fig:relu_fset} implies that it is even possible that, roughly,
$\|\proj_{\fset}(\vu+\bm{\varepsilon}) - \proj_{\fset}(\vu)\|\gg 0$ 
while $\|\vu+\bm{\varepsilon} - \vu\|=\varepsilon\to 0$.
Overall, this removes primary analytical mechanisms used in prior work.

In \cref{sec:squared}, we show how this property breakdown precludes global convergence, while \cref{sec:local} establishes local convergence for models initialized near the intersection.


\subsection{Lower Bound: Forgetting Can Be Catastrophic}
\label{sec:squared}

\citet{evron2022catastrophic} proposed that forgetting is truly ``catastrophic'' only when it fails to converge to zero in the infinite-task-sequence limit. 
In their continual linear regression setting, they demonstrated that cyclic or random task orderings eliminate forgetting as the number of iterations $k \to \infty$. Similar convergence was later established for continual linear classification \citep{evron2023classification}.

While significant, these results are expected from a projection perspective. 
Both works prove an equivalence to sequential projections onto \emph{convex} sets---specifically affine subspaces in regression and polyhedral cones in classification (see \cref{fig:linear_fset}). 
Such processes are well-known to converge to the intersection under cyclic or random orderings \citep{deutsch2006ratePOCS_I,nedic2010randomPOCS}. 
We now ask: \linebreak
\textit{Can forgetting be catastrophic in homogeneous models whose projections are nonconvex?}


We answer in the affirmative. 
We show that the sequential margin projections of \cref{alg:sequential_mm} may fail to converge even as $k \to \infty$. 
This means that continual learning on homogeneous models can forget \emph{catastrophically}, 
\ie never reaching a configuration that satisfies all tasks simultaneously.

Technically, we employ a standard 2-positively-homogeneous model that is linear in the data but nonlinear in its parameters \citep[e.g., as in][]{woodworth2020kernel}, as illustrated in \cref{fig:squared_fset}.

\begin{definition}[Squared Linear Model] \label{eq:homogeneous_model_uv}
For inputs $\x\in\R^d$, we define a 2-positively-homogeneous model parameterized by $\piterate{}=(\vu,\vv)$ where $\vu,\vv\in\R^d$:
\begin{align*}
f\bigl(\x;\vu,\vv\bigr)
\;=\;
\iprod{\vu^2}{\x}
-
\iprod{\vv^2}{\x}
\;=\;
\sum_{i=1}^d 
\prn{\vecind{u}{i}^2
-
\vecind{v}{i}^2
}\vecind{x}{i}
\,.
\end{align*}
\end{definition}


We construct a specific case exhibiting catastrophic forgetting with as few as $M=2$ tasks and $p=4$ learnable parameters.
This highlights a major qualitative gap between convex and nonconvex projection interpretations of continual learning: whereas for linear models, catastrophic forgetting arises only in the limit of $M \to \infty$ tasks \citep{evron2022catastrophic,evron2023classification}, our nonconvex setting may fail with minimal complexity.
Moreover, unlike the convex case---where task repetition drives iterates toward joint feasibility---alternating between $M=2$ tasks here does \emph{not} lead to convergence.

\begin{theorem}[Catastrophic Forgetting in Squared Linear Models]
\label{thm:lower}
Let $f(\x;\piterate{})$ be a squared linear model in $d=2$ parameterized by $\piterate{}=(\vu,\vv)$ where $\vu,\vv\in\R^{2}$.
Fix $\varepsilon\in(0,0.01)$ and define two tasks
\[
\X^{(1)}=\big[\x_1^{\top}\big]
=\big[\bigl(\tfrac19+\varepsilon,\ 1-\varepsilon\bigr)\big],
\qquad
\X^{(2)}=\big[\x_2^{\top}\big]
=\big[\bigl(\tfrac19-\varepsilon,\ -1-\varepsilon\bigr)\big],
\qquad 
\y^{(1)}=\y^{(2)}=\prn{1}\,.
\]
%
\begin{itemize}[leftmargin=0.7cm, itemindent=0cm, itemsep=0pt,labelsep=0.2cm,topsep=4pt]
\item[(a)] 
The two tasks are jointly separable, 
\ie $\interset\triangleq\fset_1\cap \fset_2 \neq\varnothing$
where the feasible sets are
$\fset_1=\{\piterate{} \in \R^4 \mid f(\x_1;\piterate{})\geq 1\},\, \fset_2=\{\piterate{} \in \R^4 \mid f(\x_2;\piterate{})\geq 1\}$.
%
\item[(b)] \cref{alg:sequential_mm} forgets catastrophically: 
$
d\bigl(\Bar{\Theta}_t,\interset\bigr)
 \ge 2,
\,\max_{m\in \{1,2\}}F_m(\Bar{\Theta}_{t})\geq 1.5,
~~\forall t \ge 0$.
\end{itemize}
\end{theorem}
The proof, provided in \cref{app:proofs_squared}, relies on the fact that for the parameters $\bar{\Theta}=(\vu,\vv)$ to lie in the intersection $\fset_1 \cap \fset_2$, it must hold that $\vecind{u}{1} \geq 2$.
However, we show that the sequential projection dynamics, when initialized at the origin $\bar{\Theta}_0=\0$, are restricted to the subspace defined by $\vecind{u}{1}=0$ and $\vecind{v}{1}=0$. 
Namely, under the specific geometry of our construction, each projection step modifies only $\vecind{u}{2}$ and $\vecind{v}{2}$, leaving the other components at zero. 
Consequently, the iterates $\bar{\Theta}_t$ are effectively trapped in a subspace bounded away from the joint feasible set, \ie $\fset_1 \cap \fset_2$.

%


Importantly, the failure in \cref{thm:lower} is not an artifact of a ``degenerate'' dataset or an adversarial task ordering: it persists over a nonzero measure of datasets under any ordering. 
Thus, global convergence guarantees require further assumptions in this nonconvex setting, as introduced next.

\pagebreak

Moreover, the lower-bound construction in \cref{thm:lower} extends beyond the case of zero initialization. In particular, a similar phenomenon can arise for arbitrary (and in particular random) initializations within a bounded ball. By appropriately adjusting the ratio between the coordinates in the construction, one can ensure that the distance to the jointly feasible set $\interset$ is arbitrarily larger than the initialization radius. In this regime, the sequential projection dynamics behave similarly to the zero-initialization case: only a subset of the coordinates is updated, even though reaching a global solution requires movement in other directions.


\subsection{Upper Bounds via Local Convergence Analysis}
\label{sec:local}

As we established, global convergence is generally unattainable in the nonconvex setting, even when individual tasks are solved to optimality.
Thus, we focus on local convergence analysis, as is common in nonconvex optimization, where global guarantees are elusive \citep[e.g.,][]{Bottou2016}.
A practical motivation is the sequential finetuning of large pretrained foundation models. These models often exhibit small zero-shot loss, implying an initialization near a joint solution.
Our analysis shows that local convergence \emph{can} be guaranteed in this regime: if the initialization is sufficiently close to the joint feasible set $\interset$, the sequence of projections converges to $\interset$ at a linear rate.


We establish this result when $f$ is $\beta$-smooth\footnote{
$f(\x;\cdot)$ is $\beta$-smooth if  $\forall \vu, \w$, 
$\big|
f(\x;\vu) - f(\x;\w) - \langle \nabla f(\x;\w), \vu - \w \rangle
\big| 
\leq \frac{\beta}{2} \bignorm{\vu - \w}^2$.
} 
and 
$G$-Lipschitz\footnote{
$f(\x;\cdot)$ is $G$-Lipschitz if $\forall \vu, \w $ , $\big|f(\x;\vu) -f(\x;\w)\big| 
\leq
G \bignorm{\vu - \w}$.
} for every data point $\x$.
In particular, both assumptions hold for DNNs with \emph{smooth} activations (\eg squared ReLU), where the Lipschitz constant may depend on architectural properties such as the network depth and homogeneity degree, as well as the magnitude of the data.
We prove the following local convergence theorem.

\begin{theorem}[Convergence Rate for Lipschitz, Smooth Homogeneous Models]
\label{thm:local_convergence}
Consider \linebreak a $G$-Lipschitz, $\beta$-$smooth$, \(\homogendeg\)-positively-homogeneous model $f(\cdot;\Bar \Theta):\mathcal X\to\mathbb R$.
Assume joint separability, \ie nonempty $\interset=
\cap_{m} \fset_{m}
\neq \emptyset$ (\Cref{asm:overparam-separability}). 
Let $\epsilon$ hold $\frac{1}{(1-\epsilon)^{M-1}} = 1 + \frac{\homogendeg^2}{2(M-1)G^2\|\opt\|^2}$ and $\delta=\hfrac{\epsilon r}{\beta \|\Theta^\star\|}$.
Then, there exists $\opt\in\interset$ 
such that if $\Bar{\Theta}_0\in \ball_\delta(\opt)$, the sequential projections of \cref{alg:sequential_mm} converge linearly with a rate depending on the task ordering, \ie  $\forall k\in \naturals$,
\begin{itemize}[leftmargin=0.5cm, itemindent=0cm, itemsep=0pt,labelsep=0.25cm,topsep=4pt]

\item \textbf{Random:}
\hfill
$\displaystyle
\frac{1}{G}
\expectation_{\ordrnd}
\max_{m\in\cnt{M}}F_{m} (\Bar{\Theta}_{k})  
~\le~
\expectation_{\ordrnd}
\dist(\Bar{\Theta}_{k}, \interset) 
~\le~
{\textstyle
\exp 
\prn{
    - \frac{k}{M}
    \frac{\homogendeg^2}{4 G^2 \|\opt\|^2}
}
}
\expectation_{\ordrnd}
\dist(\Bar \Theta_{0},\interset)$.

\item
\textbf{Cyclic ($M \mid k$):}
\hfill
$\displaystyle
\frac{1}{G}
\max_{m\in\cnt{M}}F_{m} (\Bar{\Theta}_{k})  
~\le~
\dist(\Bar{\Theta}_{k}, \interset) 
~\le~
{\textstyle
\exp 
\prn{
    - \frac{k}{M^2}
    \frac{\homogendeg^2}{4 G^2 \|\opt\|^2}
}
}
\,
\dist(\Bar \Theta_{0},\interset)
$.

\end{itemize}
\end{theorem}
The proof is provided in \cref{app:proofs_local}. 
Here, we make a few remarks and then outline the proof idea.

\begin{remark}[Comparison to Linear Model]
For a linear model, the local convergence rates above recover the global rates of \cite{evron2023classification} for both cyclic and random task orderings.
\end{remark}

\begin{remark}[Nonzero Initialization]
Local convergence requires the initialization of \cref{alg:sequential_mm} to lie within a neighborhood $\ball_\delta(\opt)$. While we prove \cref{thm:gen_m_convergence} for the case of zero initialization, this does not necessitate that $\opt$ is near the origin. Indeed, by using a technical modification, it is possible to extend the analysis to other fixed initializations where the initialization of \cref{alg:sequential_mm} scales according to the weak regularization’s $\lambda$. This could be the case, for instance, when \cref{alg:sequential_mm} is initialized with a model pretrained using weight decay \emph{with the same order of $\lambda$}, a change that does not alter the underlying geometric intuition.
\end{remark}

\begin{remark}[Order of Limits]
The number of iterations $k$ may be arbitrarily large to guarantee convergence, potentially requiring task repetitions.
Following \citep[][Remark~4.6]{evron2023classification}, we clarify that our analysis takes the limit $\lambda\downarrow 0$ \emph{after} fixing $k$.
\end{remark}

\paragraph{Proof sketch.}
As in the convex case where proofs are based on properties of POCS \citep{evron2023classification},
our proof of \cref{thm:local_convergence} leverages properties of projections on \emph{nonconvex} sets.
However, as discussed, nonconvex projections can be ill-behaved; they are often nonunique and tiny perturbations can cause large ``jumps''.
As a result, nonconvex projections are typically analyzed under regularity conditions that rule out pathological geometry and ensure the projection is locally stable, which is essential for any convergence guarantees \cite[e.g.,][]{lewis2008alternating,lewis2009local,dao2019linear}.
The resulting proof consists of two steps.

\begin{enumerate}[leftmargin=0.6cm, itemindent=0cm, itemsep=0pt,labelsep=0.2cm,topsep=4pt]
\item \textbf{First step: Smoothness and Lipschitz continuity imply regularity.}
We show that a smooth and $G$-Lipschitz homogeneous model satisfies the regularity properties required for local convergence. 
We rely on tools from variational analysis and set regularity \cite[e.g.,][]{rockafellar1998variational} to assert that the feasible sets induced by homogeneous DNNs satisfy the needed local regularity conditions. 
In particular, we establish the following two conditions.

\begin{definition}[($\epsilon,\delta)$-Regularity]
\label{def:epsdel_reg}
Let $\fset$ be a nonempty subset of $\R^\paramdim$, $\w \in \R^\paramdim$, $\varepsilon \geq 0$ and $\delta >0$.
Let $\proj_\fset(\w)$ project $\w$ onto $\fset$ and define
$
N_\fset^{\mathrm{prox}}(\w) \triangleq \left\{ \lambda(\vz - \w) \;\middle|\; \vz \in \proj_\fset^{-1}(\w),\; \lambda \geq 0 \right\}
$.
We say that $\fset$ is {\it $(\varepsilon,\delta)$-regular at $\w$} if
\begin{equation*}
\x, \vy \in \fset \cap \mathbb{B}(\w;\delta),\;
\vu \in N_\fset^{\mathrm{prox}}(\x)
\quad\Longrightarrow\quad
\langle \vu, \x - \vy \rangle \le \varepsilon \|\vu\| \cdot \|\x - \vy\|.
\end{equation*}
\end{definition}
\begin{definition}[$\kappa$-Linear Regularity of Set Collection]
\label{def:lin_reg}
A system $\{\fset_i\}_{i\in I}$ is $\kappa$-linearly regular on a subset $U \subseteq X$ if ,
for all $x\in U$,
$
\dist(x, \cap_{i\in I} \fset_i) \le \kappa \max_{i \in I} \dist(x, \fset_i)
\,. 
$
\end{definition}
In \cref{lem:epsdel_reg,lem:kap_reg} in  \cref{app:proofs_local} we show that, in the smooth and Lipschitz case, homogeneous models satisfy the conditions given in \cref{def:epsdel_reg,def:lin_reg}.

\item \textbf{Second step: Invoke convergence under regularity.}
We invoke a local linear convergence result for nonconvex projections. 
In particular, the following proposition shows that under the conditions of  \cref{def:epsdel_reg,def:lin_reg}, 
the iterates of \Cref{alg:sequential_mm} exhibit local convergence to the intersection $\interset$ both in random and cyclic orderings.
For random ordering, the proof appears in \cref{app:proofs_local}, while the cyclic case follows directly from \citet[][Corollary 5.10]{dao2019linear}.

%
%


\begin{proposition}[Stepwise decrease under random and cyclic projections]
\label{prop:one_step_random_projection}
Consider closed sets $\fset_1,\dots,\fset_M \subset \mathbb{R}^\paramdim$ with nonempty intersection
\(
\interset \triangleq \bigcap_{m=1}^M \fset_m \neq \varnothing,
\)
and let $\opt \in \interset$.
Assume that there exist constants $\kappa>0$, $\varepsilon\in[0,1)$, and $\delta>0$ such that:
\begin{enumerate}[leftmargin=0.7cm, itemindent=0cm, itemsep=0pt,labelsep=0.2cm,topsep=2pt]
\item Each set $\fset_m$ is \emph{($\epsilon,\delta$)-regular} at $\opt$.
\item The collection $\{\fset_m\}_{m=1}^M$ is \emph{$\kappa$-linearly regular} on the ball
$
\ball_{\delta/2}(\opt)
$.
\end{enumerate}
Then, we get one-step decrease depending on the task ordering,
\begin{enumerate}[leftmargin=0.7cm, itemindent=0cm,labelsep=0.2cm,topsep=4pt]
\item
\textbf{Random ordering.}
 Let $\tau=\ordrnd,t\ge1$.
If $\Bar \Theta_{t-1}\in \ball_{\delta/2}(\opt)$ and \(\rho_{\text{iid}}\triangleq\frac{1}{1-\varepsilon} - \frac{1}{M\kappa^2} < 1,
\) 
\linebreak
it holds that
\(
\mathbb E_{\tau(t)}\!\left[\dist^2(\piterate{t}, \interset)\big| \Bar \Theta_{t-1}\right]
\;\le\;
\rho_{\text{iid}}\;
\dist^2(\Bar \Theta_{t-1},\interset) .
\)

\item
\textbf{Cyclic ordering.}
Let $\tau = \ordcyc$ 
and $t$ such that $M \mid t$.
If, $\Bar \Theta_{t}\in \ball_{\delta/2}(\opt)$ and
\linebreak
$
\rho_c
\triangleq
\left(
\frac{1}{(1-\varepsilon)^{M-1}}
-
\frac{1}{(M-1)\kappa^2}
\right)^{\frac{1}{2(M-1)}}<1$, then, 
$\dist(\Bar \Theta_{t+M},\fset)
\;\le\;
\rho_{\text{cyc}}\;\dist(\Bar\Theta_t,\fset)$.
\end{enumerate}
\end{proposition}
\end{enumerate}

%% file: 05_regression.tex
\section{Extension: Continual Regression in Homogeneous DNNs as Sequential Projections}
\label{sec:regression}

We now show that the techniques developed above are general enough to extend naturally to continual regression. This extension requires only minor adjustments to the definitions.

\paragraph{Notational adjustments.}
For a task $m\in\cnt{M}$, the label vector is no longer binary, but rather $\y^{(m)}\in\R^{n_m}$.
Consequently, the feasible sets are now defined as,
\begin{equation}
\label{def:sets_reg}
\fset_m \triangleq
\Big\{\Theta\in \R^\paramdim 
~\Big|~
f(\x_i^{(m)};\Theta)
= y_i^{(m)},
\ \forall i \in \cnt{n_{m}}
\Big\},
\qquad
\forall m \in \cnt{M}\,.
\end{equation}
Accordingly, the joint separability assumption (\cref{asm:overparam-separability}) is now a joint \emph{realizability} assumption, still requiring $\interset\triangleq \fset_1\cap\dots\cap\fset_{M} \neq \emptyset$.


Furthermore, instead of the logistic loss, we now optimize the mean squared loss, \ie
\begin{equation}
\label{def:squared_loss}
\mathcal L_m(\Theta)
\triangleq
\frac{1}{n_m}
\sum_{i=1}^{n_m}
\Bigprn{
f(\x_i^{(m)};\Theta)-y_i^{(m)}
}^2,
\qquad
\forall m\in\cnt{M}
\,.
\end{equation}
We still consider $\dist(\cdot, \fset_{m})$ and $\dist(\cdot, \interset)$ (as in \cref{sec:setting_classification}),
but redefine the \emph{forgetting} with respect to the squared loss instead of the hinge loss,
\ie 
\begin{equation}
    F_{\tau(t')}(\piterate{t})
    =
    \max_i 
    \Bigprn{
    f(\x_i^{\tau(t')};\piterate{t})-y_i^{\tau(t')}
    }^2
    ,
    \qquad
    \forall t \in \cnt{k}, \forall t' \le t
    \,.
\end{equation}

\paragraph{Result.}
Analogously to our classification result, we establish the following theorem for the convergence of \cref{alg:regularized_cl} to \cref{alg:sequential_mm} as $\lambda\downarrow 0$ for the regression case.

\begin{theorem}[Weakly-Regularized Continual Regression $\to$ Sequential Projections]
\label{thm:gen_m_convergence_reg}
Consider a model $f(\cdot;\Theta) : \mathcal{X} \to \mathbb{R}$.
Assume individual realizability, \ie nonempty feasible sets $\fset_{1}, \dots, \fset_{M}$ (\Cref{def:sets_reg}). 
Then, as $\lambda\downarrow 0 $, \cref{alg:regularized_cl} trained with the squared loss aligns with \Cref{alg:sequential_mm}.
That is, for every iteration $t \in [k]$, any sequence $\lambda \downarrow 0$ admits a subsequence $(\lambda_j)$ and a point $\Bar{\Theta}_t\in \proj_{\tau(t)}\tprn{\piterate{t-1}}\triangleq
    \argmin_{\bar{\Theta}
    \in\mathcal \fset_{\tau(t)}}
    \bignorm{\bar{\Theta}-\piterate{t-1}}^2$ such that
\(
    \rjiterate{t}
    ~
    \xrightarrow{\quad j \to \infty \quad} 
    ~
    \piterate{t}
    \in
    \proj\tprn{\piterate{t-1}}
    \,.
\)
\end{theorem}
Proofs for this section appear in \cref{app:regression}; as with classification, the analysis uses $\Gamma$-convergence.


\subsection{Convergence for Overparameterized Homogeneous Regression Models}

Mirroring the classification case (\cref{thm:lower}),
we first show that the $2$-positively-homogeneous squared model of \cref{eq:homogeneous_model_uv} may suffer from \emph{catastrophic} forgetting.

\begin{lemma}[Catastrophic Forgetting in Squared Models: Simplified Version]
\label{thm:lower_equality}
Under the same task construction of \cref{thm:lower}, 
the tasks are jointly separable, yet \cref{alg:sequential_mm} forgets catastrophically.
That is, $\interset\neq\emptyset$ but
$\dist\bigl(\Bar{\Theta}_t,\interset\bigr)
\ \ge\ 2$ and $\max_m F_{m}(\Bar{\Theta}_{t})\geq 3$,
for every iteration $t\ge0$.
\end{lemma}

The formal statement and its proof appear in \cref{app:regression}. 
Specifically, we show that in the \emph{classification} case (\cref{thm:lower}), the current iterate $\Bar{\Theta}_t$ always lies on the boundary of the latest feasible set; thus, the dynamics of \cref{alg:sequential_mm} is identical for regression and classification.


After establishing that global convergence is not guaranteed in the general homogeneous regression setting, we now show that local convergence can still be achieved. In contrast to our classification result, the local convergence result for regression requires the additional assumption
\(
y_{\min}=\min_{m,i}|y_i^{(m)}| \neq 0.
\)
This assumption is imposed to ensure that the $(\epsilon,\delta)$-regularity condition holds (see \cref{lem:homo_eps_delta} for more details).

\begin{theorem}[Convergence Rate for Lipschitz, Smooth Homogeneous Models]
\label{thm:local_convergence_reg}
Consider
\linebreak a
$G$-Lipschitz, $\beta$-$smooth$,
\(\homogendeg\)-positively-homogeneous model $f(\cdot;\Bar \Theta):\mathcal X\to\mathbb R$.
Assume joint realizability, \ie nonempty intersection $\interset=
\cap_{m} \fset_{m}
\neq \emptyset$.
Let $y_{\min}=\min_{m,i} |y_i^{(m)}|\neq 0$.
Let $\epsilon$ hold $\frac{1}{(1-\epsilon)^{M-1}} = 1 + \frac{\homogendeg^2}{2(M-1)G^2\|\opt\|^2}$ and $\delta=\frac{y_{\min}\epsilon r}{\beta \|\Theta^\star\|}$.
Then, there exists $\opt\in\interset$ 
s.t.\ if $\Bar{\Theta}_0\in \ball_\delta(\opt)$, the sequential projections of \cref{alg:sequential_mm} converge linearly with a rate depending on the task ordering, \ie  $\forall k\in \naturals$,
\begin{itemize}[leftmargin=0.5cm, itemindent=0cm, itemsep=0pt,labelsep=0.25cm,topsep=4pt]

\item \textbf{Random:}
\hfill
$\displaystyle
\frac{1}{G}
\expectation_{\ordrnd}
\max_{m\in\cnt{M}}F_{m} (\Bar{\Theta}_{k})  
~\le~
\expectation_{\ordrnd}
\dist^2(\Bar{\Theta}_{k}, \interset) 
~\le~
{\textstyle
\exp 
\prn{
    - \frac{k}{M^2}
    \frac{\homogendeg^2y^2_{\min}}{2 G^2 \|\opt\|^2}
}
}
\expectation_{\ordrnd}
\dist^2(\Bar \Theta_{0},\interset)$.

\item
\textbf{Cyclic ($M \mid k$):}
\hfill
$\displaystyle
\frac{1}{G}
\max_{m\in\cnt{M}}F_{m} (\Bar{\Theta}_{k})  
~\le~
\dist^2(\Bar{\Theta}_{k}, \interset) 
~\le~
{\textstyle
\exp 
\prn{
    - \frac{k}{M^2}
    \frac{\homogendeg^2y^2_{\min}}{2 G^2 \|\opt\|^2}
}
}
\,
\dist^2(\Bar \Theta_{0},\interset)
$.

\end{itemize}
\end{theorem}

%% file: 90_appendices.tex
\appendix
\crefalias{section}{appendix} 

\renewcommand{\thetheorem}{\Alph{section}.\arabic{theorem}}
\renewcommand{\thelemma}{\Alph{section}.\arabic{lemma}}

\makeatletter
\@addtoreset{theorem}{section}
\makeatother

\startcontents[app]

\begingroup
\renewcommand{\baselinestretch}{1.5}\selectfont
\appendixtableofcontents
\endgroup

\newpage

\newpage
\newpage
\input{93_appendix_convergence}

\newpage
\input{94_appendix_squared}

\newpage
\input{95_appendix_local}
\newpage
\input{97_appendix_regression}



%% file: 93_appendix_convergence.tex
\section{Proofs for \cref{sec:main_thm}}
\label{app:proof_main}

\begin{proof} [of \cref{thm:gen_m_convergence}]
For any $\lambda\in(0,1)$, we define $c_{\lambda,\homogendeg}=\left(\log \frac{1}{\lambda}\right)^\frac{1}{\homogendeg}\,.$ 
We prove in induction that for every $t$, 
every sequence $\lambda\downarrow 0$ admits a subsequence $(\lambda_j)$ and a point
$\Bar\Theta_t\in\fset_{\tau(t)}$ such that
\begin{equation*}
\frac{\Theta_t^{(\lambda_j)}}{c_{\lambda_j,\homogendeg}}\ \longrightarrow\ \Bar{\Theta}_t
\qquad\text{for}\qquad
\bar{\Theta}_t\in\argmin_{\Theta
                    \in\mathcal \fset_{\tau(t)}}
                    \bignorm{\Theta-\piterate{t-1}}^2\,.
\end{equation*}

The base case, $t=0$ holds trivially for the constant sequence $\Theta_0^{\lambda}=0=\Bar{\Theta}_0$. For the step,  let $(\lambda)$ be a sequence that converges to zero. We assume that there exists a subsequence $(\lambda_j)$ and a point
$\piterate{t-1}\in\fset_{\tau(t-1)}$ that
holds
\begin{equation*}
\frac{\Theta_{t-1}^{(\lambda_j)}}{c_{\lambda_j,\homogendeg}}\ \longrightarrow\ \Bar{\Theta}_{t-1}
\qquad\text{for}\qquad
\bar{\Theta}_{t-1}\in\argmin_{\Theta
                    \in\mathcal \fset_{\tau({t-1})}}
                    \bignorm{\Theta-\piterate{t-2}}^2\,.
\end{equation*}
For every $\lambda$, we define $$\hat{\Theta}_t^{(\lambda)}\triangleq\frac{\Theta_t^{(\lambda)}}{c_{\lambda,\homogendeg}},\quad
G_t^{(\lambda)}(\Theta)\triangleq\frac{\mathcal L_{\tau(t)}(c_{\lambda,\homogendeg} \Theta)}{c_{\lambda,r}^2\lambda }
+\Bigl\|\Theta-\hat\Theta_{t-1}^{(\lambda)}\Bigr\|^2\,.
$$
Then, for every $\lambda>0$, 
\begin{align*}
    &\Theta_t^{(\lambda)}\in \arg \min \mathcal L_{\tau(t)}(\Theta)+\lambda\|\Theta-\Theta_{t-1}^{(\lambda)}\|^2
    \\&\iff\Theta_t^{(\lambda)}\in \arg \min \frac{1}{\lambda}\mathcal L_{\tau(t)}(\Theta)+\|\Theta-\Theta_{t-1}^{(\lambda)}\|^2
    \\&\iff\hat\Theta_t^{(\lambda)}\in \arg \min \frac{1}{\lambda}\mathcal L_{\tau(t)}(c_{\lambda,\homogendeg}\Theta)+\|c_{\lambda,\homogendeg}\Theta-\Theta_{t-1}^{(\lambda)}\|^2
    \\&\iff\hat\Theta_t^{(\lambda)}\in \arg \min \frac{1}{\lambda c_{\lambda,\homogendeg}^2}\mathcal L_{\tau(t)}(c_{\lambda,\homogendeg}\Theta)+\|\Theta-\hat\Theta_{t-1}^{(\lambda)}\|^2
    \\&\iff \hat{\Theta}_t^{(\lambda)} \in \arg \min G_t^{(\lambda)}(\Theta).
\end{align*}
Now, we look at the sequence $\hat \Theta_t^{(\lambda_j)}$.
For proving the required, it is sufficient to show a sub-sequence of $\hat\Theta_t^{(\lambda_j)}$ that converges to $\Bar{\Theta}_t\in\argmin_{\Theta
                    \in\mathcal \fset_{\tau({t})}}
                    \bignorm{\Theta-\piterate{t-1}}^2$.
Thus, by, \cref{thm:fundamental_gamma}, 
it is sufficient to prove that the sequence $G_t^{(\lambda_j)}$ is equi-coercive and that this sequence $\Gamma$-converges to some function $G_t$ with
$\Bar\Theta_t\in \arg \min G_t$
when $j$ goes to $\infty$.

\paragraph{$\Gamma$-convergence.} For the $\Gamma$-convergence, let \(G_t(\Theta)=\mathbf 1_{\fset_{\tau(t)}}(\Theta)+\|\Theta-\bar\Theta_{t-1}\|^2,\) where $ 1_{\fset_{\tau(t)}}({\Theta})=0$ if $\Theta\in \fset_{\tau(t)}$ and $\infty$ otherwise.
\begin{itemize}[leftmargin=0.45cm, itemindent=0cm, itemsep=0pt,labelsep=0.2cm,topsep=4pt]
    \item 
For the \emph{liminf} property, let $\Bar\Theta$ and a sequence $\hat{\Theta}^{(\lambda_j)}\to \Bar\Theta$. We need to prove that
\[
G_t(\Bar{\Theta})\le \liminf_{j\to\infty} G_t^{(\lambda_j)}(\hat{\Theta}^{(\lambda_j)}).
\]
For the second term of the function, by the induction hypothesis and continuity of norm, 
$\bigl\|\hat\Theta_t^{(\lambda_j)}-\hat\Theta_{t-1}^{(\lambda_j)}\bigr\|^2\to \|\Bar\Theta-\Bar\Theta_{t-1}\|^2$.
For the first term of the function, since $\mathcal L_{\tau(t)}\geq 0$, \[
\liminf_{j\to \infty}\ \frac{\mathcal L_{\tau(t)}(c_{\lambda_j,\homogendeg} \hat\Theta^{(\lambda_j)})}{\lambda_j\,c_{\lambda_j,\homogendeg}^2}\ \ge\ 0.
\]
Combining both together we get that if $\Bar{\Theta}\in \fset_{\tau(t)}$ then
\begin{align*}
G_t(\Bar{\Theta})
&=
\mathbf 1_{\fset_\tau(t)}(\Bar{\Theta})+\|\bar\Theta-\bar\Theta_{t-1}\|^2
=
\|\bar\Theta-\bar\Theta_{t-1}\|^2
=
\liminf_{j\to\infty}\bigl\|\hat\Theta_t^{(\lambda_j)}-\hat\Theta_{t-1}^{(\lambda_j)}\bigr\|^2
\\&
\leq
\liminf_{j\to \infty}\ \frac{\mathcal L_{\tau(t)}(c_{\lambda_j,\homogendeg} \hat\Theta^{(\lambda_j)})}{\lambda_j\,c_{\lambda_j,\homogendeg}^2}
+\Bigl\|\hat{\Theta}^{(\lambda_j)}-\hat\Theta_{t-1}^{(\lambda_j)}\Bigr\|^2
\\
&=\liminf_{j\to\infty}G_t^{(\lambda_j)}(\hat{\Theta}^{(\lambda_j)})\,.
\end{align*}

If $\Bar{\Theta}\notin \fset_{\tau(t)}$, let
$u_{\lambda,i}:=c_{\lambda,\homogendeg}^\homogendeg\,y_i^{(\tau(t))} f(x_i^{(\tau(t))};\hat\Theta^{(\lambda)})$ and consider the index realizing $\min_i u_{\lambda,i}$.
Since $\Bar{\Theta}\notin \fset_{\tau(t)}$, we know that this index $i$ satisfies $u_{\lambda,i}<c_{\lambda,\homogendeg}^\homogendeg$.
First consider the case where $u_{\lambda_{j},i}\le 0$ infinitely often in the sequence. Then by using the property that if $u\le 0$ then 
$\log(1+e^{-u})\ge \log 2$, it holds for this sub-sequence that  
$$\mathcal L_{\tau(t)}(c_{\lambda_{j},a} \hat\Theta^{(\lambda_j)})\ge (\log 2)/n_{\tau(t)}\,,$$
and therefore, as for $\Bar{\Theta}\notin \fset_{\tau(t)}$ we have that \(G_t(\Bar{\Theta})=\infty\), we get that
$$
\liminf_{j\to\infty} G_t^{(\lambda_j)}(\hat{\Theta}^{(\lambda_j)})\geq\mathcal L_{\tau(t)}(c_{\lambda,\homogendeg} \hat\Theta^{(\lambda_j)})/(\lambda_j c_{\lambda_j,\homogendeg}^2)\to\infty=G_t(\Bar{\Theta})\,.$$
 For the case where it does not hold that $u_{\lambda_{j},i}\le 0$ infinitely often in the sequence, there exists $\delta\in(0,1)$ s.t.\ $0<u_{\lambda_j,i}\le c_{\lambda_j,\homogendeg}^\homogendeg(1-\delta)$. 
 As a result, since $e^{-c_{\lambda_j,\homogendeg}^r}=e^{-\log(1/\lambda_j)}=\lambda_j$, we get
\begin{align*}
&\mathcal L_{\tau(t)}(c_{\lambda_j,\homogendeg} \hat\Theta^{(\lambda_j)})\ge \frac{1}{2n_{\tau(t)}}e^{-u_{\lambda_j,i}}\ge \frac{1}{2n_{\tau(t)}}e^{-c_{\lambda_j,\homogendeg}^\homogendeg(1-\delta)}\ge \frac{1}{2n_{\tau(t)}}\lambda_j^{1-\delta} \\&\implies
\frac{\mathcal L_{\tau(t)}(c_{\lambda_j,\homogendeg} \hat\Theta^{(\lambda_j)})}{(\lambda c_{\lambda_j,\homogendeg}^2)}\ge \frac{1}{2n_{\tau(t)}}\frac{\lambda_j^{-\delta}}{c_{\lambda_j,\homogendeg}^2}\to\infty.
\end{align*}\

\item For the \emph{Limsup}, let $\Bar{\Theta}$. 
If $\Bar{\Theta}\in \fset_{\tau(t)}$, let $\hat\Theta^{(\lambda_j)}=\Bar{\Theta}$ for every $j$ (constant sequence). Then, for every $i$, $y_i^{(\tau(t))} f(\x_i^{(\tau(t))};\bar\Theta)\ge 1$, and,
\begin{align*}
\mathcal L_{\tau(t)}(c_{\lambda_j,\homogendeg} \bar\Theta)
&=\frac1{n_{\tau(t)}}\sum_i \log\!\bigl(1+e^{-\,c_{\lambda_j,\homogendeg}^\homogendeg\,y_i^{(m)} f(\x_i^{(m)};\Theta)}\bigr)
\\&\le \frac1{n_{\tau(t)}}\sum_i e^{-\,c_{\lambda_j,\homogendeg}^\homogendeg\,y_i^{(m)} f(\x_i^{(m)};\Theta)}
\le e^{-\,c_{\lambda_j,\homogendeg}^\homogendeg}
=\lambda_j 
\,.  
\end{align*}
Hence $\mathcal L_{\tau(t)}(c_{\lambda_j,\homogendeg} \Bar{\Theta})/(\lambda_j c_{\lambda_j,\homogendeg}^2)\le 1/c_{\lambda_j,\homogendeg}^2\to 0$ as $j\to \infty$.
Then, by the induction hypothesis and continuity of norm, $G_t^{(\lambda_j)}(\Bar\Theta)\to \|\Bar\Theta-\Bar{\Theta}_{t-1}\|^2=G_t(\Bar\Theta)$.
Otherwise, $\Bar\Theta\notin\fset_{\tau(t)}$. Then $G_t(\Bar \Theta)=+\infty$, so the \(\Gamma\)-limsup property is trivial.

\end{itemize}

\paragraph{Boundedness of level sets.}
By the positivity of $\mathcal L_{\tau(t)}$, for all $\Bar \Theta$, \[
G_t^{(\lambda_j)}(\Bar\Theta)\ \ge\ \Bigl\|\Bar \Theta-\hat\Theta_{t-1}^{(\lambda_j)}\Bigr\|^2.
\]
Thus every sublevel set $\{\bar\Theta:\ G_t^{(\lambda_j)}(\Bar{\Theta)}\le C\}$ is contained in the closed ball centered at
$\hat\Theta_{t-1}^{(\lambda_j)}$ of radius $\sqrt{C}$. 
These balls are compact sets and the equi-coercivity follows.

By \cref{thm:fundamental_gamma},
there is a subsequence $\hat{\Theta}_t^{(\lambda_{j_{\ell}})} \to \bar{\Theta}_t$. For getting the convergence in direction,
    note that $\Bar{\Theta}_t\in \proj_{\tau(t)}\tprn{\piterate{t-1}}$ (as a minimizer of $G_t$), which implies $y_i^{\tau(t)} f(\x_i^{(\tau(t))}; \bar{\Theta}_t) \ge 1$ for all $i$. Therefore, $\bar{\Theta}_t \neq 0$ (since by homogeneity $f(\cdot; 0) = 0$).
    Since the limit is nonzero, the mapping $u \mapsto u/\|u\|$ is continuous at $\bar{\Theta}_t$. As a result,
    \[
    \frac{\Theta_t^{(\lambda_{j_{\ell}})}}{\left\|\Theta_t^{(\lambda_{j_{\ell}})}\right\|} 
    = \frac{c_{\lambda_{j_{\ell}}}\hat{\Theta}_t^{(\lambda_{j_{\ell}})}}{\left\|c_{\lambda_{j_{\ell}}}\hat{\Theta}_t^{(\lambda_{j_{\ell}})}\right\|}
    = \frac{\hat{\Theta}_t^{(\lambda_{j_{\ell}})}}{\left\|\hat{\Theta}_t^{(\lambda_{j_{\ell}})}\right\|}
    \xrightarrow{\quad j \to \infty \quad} 
    \frac{\bar{\Theta}_t}{\left\|\bar{\Theta}_t\right\|}.
    \]
    This concludes the proof.
\end{proof}
\newpage

%% file: 94_appendix_squared.tex
\section{Proofs for \cref{sec:squared}}
\label{app:proofs_squared} 
In this section, we prove the lemmas used for the proof of \cref{thm:lower}.

The proof is based on an analysis of the dynamics of \cref{alg:sequential_mm} in the construction given in \cref{thm:lower}.
We begin with several lemmas.

\begin{proposition} [Projection onto $\fset_1$ from the origin]
\label{lem:proj_ellipse_yaxis}
Let $a_1,b>0$ with $b>a_1$. Consider
\[
\min_{x,y\in\mathbb R}\; x^2+y^2
\quad\text{s.t.}\quad
a_1x^2+by^2\ge 1.
\]
Then every minimizer satisfies $x=0$ and $y=\pm b^{-1/2}$. In particular, the minimum value is $1/b$.
\end{proposition}

\begin{proof}
Let $(x,y)$ be feasible. First we show that any minimizer must satisfy
\[
a_1x^2+by^2=1.
\] 
Assume in contradiction that $r= a_1x^2+by^2>1$ and set $
\alpha:=\frac{1}{\sqrt{r}}\in(0,1)$.
Then $(\alpha x,\alpha y)$ satisfies the constraint with equality:
\[
a_1(\alpha x)^2+b(\alpha y)^2=\alpha^2(a_1x^2+by^2)=1,
\]
and the objective strictly decreases:
\[
(\alpha x)^2+(\alpha y)^2=\alpha^2(x^2+y^2)<x^2+y^2,
\]
implying that $r=1$.

Second,on the boundary $a_1x^2+by^2=1$ we have
\[
y^2=\frac{1-a_1x^2}{b},
\qquad\text{with } 0\le a_1x^2\le 1 \ \Longleftrightarrow\ 0\le x^2\le \frac{1}{a_1}.
\]
Substituting into the objective gives
\[
x^2+y^2
=x^2+\frac{1-a_1x^2}{b}
=\frac{1}{b}+x^2\Bigl(1-\frac{a_1}{b}\Bigr).
\]
Since $b>a_1$, we have $1-\frac{a_1}{b}>0$, so the right-hand side is strictly increasing in $x^2$.
Therefore it is minimized when $x^2$ is minimal, i.e.\ when $x^2=0$.
With this choice of $x=0$, the active constraint implies $by^2=1$, hence $y=\pm b^{-1/2}$.
This proves the claim.
\end{proof}

\begin{proposition}
[Projection onto $\fset_2$ from the $y$--axis]
\label{lem:proj_K2_explicit}
Let $a_2,b,c>0$ with $c>a_2$ and let $u=(0,b^{-1/2},0)$.
Consider the optimization problem
\[
\min_{(x,y,z)\in\mathbb R^3}
F(x,y,z):=x^2+(y-b^{-1/2})^2+z^2
\quad\text{s.t.}\quad
a_2x^2-cy^2+cz^2\ge1.
\]
Then every minimizer satisfies
\[
x=0,\qquad y=\pm\frac{1}{2\sqrt b},
\qquad
z=\pm\sqrt{\frac{1}{4b}+\frac1c}.
\]
In particular, the set of minimizers is
\[
B:=\left\{\left(0,\ \pm\frac{1}{2\sqrt b},\ \pm\sqrt{\frac{1}{4b}+\frac1c}\right)\right\},
\]
and every $\Theta^{(2)}=(0,y_0,z_0)\in B$ satisfies
\[
z_0^2-y_0^2=\frac1c.
\]
\end{proposition}

\begin{proof}
First, we show that for any solution of the problem, the constraint is active.
Let $(x,y,z)$ be feasible with $a_2x^2-cy^2+cz^2>1$.
Scaling $(x,y,z)\mapsto\alpha(x,y,z)$ with
\[
\alpha:=\frac{1}{\sqrt{a_2x^2-cy^2+cz^2}}\in(0,1)
\]
preserves feasibility with equality and strictly decreases $F$ since $F(\alpha\Bar{\Theta})=\alpha^2 \Bar{\Theta}$ for every $\Bar{\Theta}$.
Thus every minimizer satisfies
\[
a_2x^2-cy^2+cz^2=1.
\]

Next,
from feasibility we have $
cz^2\ge1+cy^2-a_2x^2$.
Substituting into $F$ yields
\[
F(x,y,z)\ge
x^2+(y-b^{-1/2})^2+\frac{1+cy^2-a_2x^2}{c}
=
\Bigl(1-\frac{a_2}{c}\Bigr)x^2+(y-b^{-1/2})^2+y^2+\frac1c.
\]
Since function $(y-b^{-1/2})^2+y^2$ is strictly convex and is minimized at
$y=\tfrac12 b^{-1/2}$, where it takes value $\tfrac{1}{2b}$.
Hence every feasible point satisfies
\[
F(x,y,z)\ge
\Bigl(1-\frac{a_2}{c}\Bigr)x^2+\frac{1}{2b}+\frac1c.
\]
Equality is achieved when $x=0$, $y=\tfrac12 b^{-1/2}$, and
\[
cz^2=1+cy^2
\quad\Longleftrightarrow\quad
z^2=\frac{1}{4b}+\frac1c.
\]
This yields the claimed minimizers.

The final identity $z_0^2-y_0^2=\frac1c$ follows by direct substitution.
\end{proof}

\begin{proposition}[Projections onto $\fset_1$]
\label{lem:kkt_two_branches_detailed}
Let $
\fset_1:=\Bigl\{(x,y,z)\in\mathbb R^3:\ a_1 x^2 + b y^2 - b z^2 \ge 1\Bigr\},
\fset_2:=\Bigl\{(x,y,z)\in\mathbb R^3:\ a_2 x^2 - c y^2 + c z^2 \ge 1\Bigr\} 
$ for  $
a_1:=\tfrac19+\varepsilon,\quad b:=1-\varepsilon,
a_2:=\tfrac19-\varepsilon,\quad c:=1+\varepsilon
$.
Fix $\bar{\Theta}_t=(0,y_0,z_0)\in\mathbb \fset_2$ and consider the projection problem
\begin{equation*}
\label{eq:proj_C1_eps}
\min_{(x,y,z)\in\mathbb R^3}\;
F(x,y,z):=x^2+(y-y_0)^2+(z-z_0)^2
\quad\text{s.t.}\quad
g(x,y,z):=a_1 x^2 + b y^2 - b z^2-1\ \ge\ 0.
\end{equation*}
Then any global minimizer $(x^\star,y^\star,z^\star)$ satisfies $x^\star=0$.
\end{proposition}

\begin{proof}
The objective $F$ and the constraint function $g$ are $C^\infty$.
Since $u\notin C_1$ and $C_1$ is closed, a minimizer exists and is not equal to $u$.
Then, every minimizer satisfies $g(x^\star,y^\star,z^\star)=0$.
On $g=0$, $\nabla g(x,y,z)=(2a_1x,2by,-2bz)\neq 0$ (otherwise $x=y=z=0$ would contradict $g(0,0,0)=-1$).
So LICQ holds and KKT is necessary.

Next, we calculate the KKT points.
We use the Lagrangian.
\[
\mathcal L(x,y,z,\lambda):=F(x,y,z)-\lambda g(x,y,z),\qquad \lambda\ge0.
\]
Stationarity gives
\[
\partial_x\mathcal L=2x-2\lambda a_1 x=0
\quad\Longleftrightarrow\quad
x(1-\lambda a_1)=0,
\]
\[
\partial_y\mathcal L=2(y-y_0)-2\lambda b y=0
\quad\Longleftrightarrow\quad
y(1-\lambda b)=y_0,
\]
\[
\partial_z\mathcal L=2(z-z_0)+2\lambda b z=0
\quad\Longleftrightarrow\quad
z(1+\lambda b)=z_0.
\]
Primal feasibility and complementary slackness imply $g(x^\star,y^\star,z^\star)=0$, thus, since $x^\star(1-\lambda a_1)=0$, either $x^\star=0$ or $\lambda=1/a_1$.

Branch I corresponds to $x=0$. In this case,
\(
y=\frac{y_0}{1-\lambda b},z=\frac{z_0}{1+\lambda b},
\)
and $g=0$ becomes $b(y^2-z^2)=1$.

Branch II corresponds to $\lambda=1/a_1$. In this case,
\(
y=\frac{y_0}{1-b/a_1}, z=\frac{z_0}{1+b/a_1},
\)
and $g=0$ yields $x^2=\frac{1-b(y^2-z^2)}{a_1}$.

Finally, it remains to show that any global minimizer will be in the first branch ($x^\star=0$).

Branch II yields an objective value of 
\[
\|(x,y,z)-\bar{\Theta}_t\|^2
 \ge\ \frac{1}{a_1},
\]
Since $a_1=\tfrac19+\varepsilon\le\tfrac19+0.01<0.13$,
\[
\frac{1}{a_1} > 7.6.
\]

In addition, branch I contains feasible point with small objective.
Consider the feasible point with $x=0$, $z=z_0$, and
\[
y=\operatorname{sign}(y_0)\sqrt{z_0^2+\frac{1}{b}},
\]
which satisfies $by^2-bz^2=1$ (hence lies on the boundary of $\fset_1$).
Then
\[
\|(x,y,z)-\bar{\Theta}_t\|^2=(\sqrt{z_0^2+1/b}-|y_0|)^2.
\]
Using $z_0^2\ge y_0^2+1/c$ and $1/b\le 1/c + 2\varepsilon$, one gets
\[
\sqrt{z_0^2+1/b}-|y_0|\ \le\ \sqrt{y_0^2+1/c+1/b}-|y_0|
\ \le\ \sqrt{2/b}\ \le\ \sqrt{2.02},
\]
hence
\[
\|(x,y,z)-\bar{\Theta}_t\|^2\ \le\ 2.02.
\]
Therefore Branch I attains objective $\le 2.02$, whereas Branch II is $\ge 1/a_1>7.6$.
So every global minimizer lies in Branch I, and in particular has $x=0$.

\end{proof}

\begin{proposition}[Projections onto $\fset_2$]
\label{lem:kkt_projection_f2}
Let $
\fset_1:=\Bigl\{(x,y,z)\in\mathbb R^3:\ a_1 x^2 + b y^2 - b z^2 \ge 1\Bigr\},
\fset_2:=\Bigl\{(x,y,z)\in\mathbb R^3:\ a_2 x^2 - c y^2 + c z^2 \ge 1\Bigr\} 
$ for  $
a_1:=\tfrac19+\varepsilon,\quad b:=1-\varepsilon,
a_2:=\tfrac19-\varepsilon,\quad c:=1+\varepsilon
$.
Fix $\bar{\Theta}_t=(0,y_0,z_0)\in\mathbb \fset_2$ and consider the projection problem onto $\fset_2$:
\begin{equation*}
\min_{(x,y,z)\in\mathbb R^3}\;
F(x,y,z):=x^2+(y-y_0)^2+(z-z_0)^2
\quad\text{s.t.}\quad
g(x,y,z):=a_2 x^2 - c y^2 + c z^2 - 1 \ge 0.
\end{equation*}
Then any global minimizer $(x^\star,y^\star,z^\star)$ satisfies $x^\star=0$.
\end{proposition}

\begin{proof}
The objective $F$ and the constraint function $g$ are $C^\infty$.
Since $\bar{\Theta}_t \in \fset_1$, we have $b y_0^2 - b z_0^2 \ge 1$, which implies $y_0^2 > z_0^2$. However, points in $\fset_2$ with $x=0$ satisfy $c z^2 - c y^2 \ge 1$, implying $z^2 > y^2$. Thus $\bar{\Theta}_t \notin \fset_2$.
Since $\fset_2$ is closed, a minimizer exists.
Then, every minimizer satisfies $g(x^\star,y^\star,z^\star)=0$.
On $g=0$, $\nabla g(x,y,z)=(2a_2x,-2cy,2cz)\neq 0$ (otherwise $x=y=z=0$ would contradict $g(0,0,0)=-1$).
So LICQ holds and KKT is necessary.

Next, we calculate the KKT points.
We use the Lagrangian.
\[
\mathcal L(x,y,z,\lambda):=F(x,y,z)-\lambda g(x,y,z),\qquad \lambda\ge0.
\]
Stationarity gives
\[
\partial_x\mathcal L=2x-2\lambda a_2 x=0
\quad\Longleftrightarrow\quad
x(1-\lambda a_2)=0,
\]
\[
\partial_y\mathcal L=2(y-y_0)+2\lambda c y=0
\quad\Longleftrightarrow\quad
y(1+\lambda c)=y_0,
\]
\[
\partial_z\mathcal L=2(z-z_0)-2\lambda c z=0
\quad\Longleftrightarrow\quad
z(1-\lambda c)=z_0.
\]
Primal feasibility and complementary slackness imply $g(x^\star,y^\star,z^\star)=0$, thus, since $x^\star(1-\lambda a_2)=0$, either $x^\star=0$ or $\lambda=1/a_2$.

Branch I corresponds to $x=0$. In this case,
\(
y=\frac{y_0}{1+\lambda c},z=\frac{z_0}{1-\lambda c},
\)
and $g=0$ becomes $c(z^2-y^2)=1$.

Branch II corresponds to $\lambda=1/a_2$. In this case,
\(
y=\frac{y_0}{1+c/a_2}, z=\frac{z_0}{1-c/a_2},
\)
and $g=0$ yields $x^2=\frac{1-c(z^2-y^2)}{a_2}$.

Finally, it remains to show that any global minimizer will be in the first branch ($x^\star=0$).
Branch II yields an objective value of 
\[
\|(x,y,z)-\bar{\Theta}_t\|^2 \ge \frac{1}{a_2}.
\]
Since $a_2=\tfrac19-\varepsilon$,
\[
\frac{1}{a_2} > 9.
\]

In addition, branch I contains feasible point with small objective.
Consider the feasible point with $x=0$, $y=y_0$, and
\[
z=\operatorname{sign}(z_0)\sqrt{y_0^2+\frac{1}{c}},
\]
which satisfies $cz^2-cy^2=1$ (hence lies on the boundary of $\fset_2$).
Then
\[
\|(x,y,z)-\bar{\Theta}_t\|^2=(\sqrt{y_0^2+1/c}-|z_0|)^2.
\]
Using $y_0^2 \ge z_0^2+1/b$ (from $\bar{\Theta}_t \in \fset_1$), the term $\sqrt{y_0^2+1/c}-|z_0|$ is maximized when $z_0=0$ and $y_0^2=1/b$. Thus
\[
\sqrt{y_0^2+1/c}-|z_0|\ \le\ \sqrt{1/b+1/c}
\ \le\ \sqrt{2.02},
\]
hence
\[
\|(x,y,z)-\bar{\Theta}_t\|^2\ \le\ 2.02.
\]
Therefore Branch I attains objective $\le 2.02$, whereas Branch II is $\ge 1/a_2 > 9$.
So every global minimizer lies in Branch I, and in particular has $x=0$.
\end{proof}

Now given the technical lemmas above we can finally prove \cref{thm:lower}.

\begin{proof} [of \cref{thm:lower}]
First for Part $(a)$, 
let $
a_1:=\tfrac19+\varepsilon,b:=1-\varepsilon,
a_2:=\tfrac19-\varepsilon,c:=1+\varepsilon.
$
In this proof we refer to $\Bar{\Theta}=(u,v)$ as a vector in $\R^4$ where its first and third entries are $u$ and the two other entries are $v$.
It holds that
\begin{align*}
    \fset_1=\{\Bar{\Theta}\in \R^4 \mid a_1\vecind{\Bar{\Theta}}{1}^2-a_1\vecind{\Bar{\Theta}}{2}^2 + b\vecind{\Bar{\Theta}}{3}^2-b\vecind{\Bar{\Theta}}{4}^2\geq 1\},
\end{align*}
\begin{align*}
    \fset_2=\{\Bar{\Theta}\in \R^4 \mid a_2\vecind{\Bar{\Theta}}{1}^2-a_2\vecind{\Bar{\Theta}}{2}^2 -c \vecind{\Bar{\Theta}}{3}^2+c\Bar{\Theta} ^2_4\geq 1\}.
\end{align*}
Then, since $\varepsilon<0.01$, $\Bar{\Theta}=(4,0,0,0)\in \fset_1\cap \fset_2$.

For part $(b)$, 
we first show that that for any sequence of projections, for every iteration $t$, $\vecind{\Bar{\Theta}_{t}}{2}=0$ by induction of $(t)$.
The basis of $t=0$ follows by initialization in the origin. For the step assume that $\vecind{\Bar{\Theta}_{t}}{2}$ and assume in contradiction that $\vecind{\Bar{\Theta}_{t+1}}{2}=\delta\neq 0$. Let $\Tilde{\Theta}=\vecind{\Bar{\Theta}_{t+1}}{2}-\delta e_2$.
It holds that $\Tilde{\Theta}_2=0$, and for every $m \in {1,2}$
\begin{align*}
    f(\Tilde{\Theta},\x_m) \geq \Bar{\Theta}_{t+1} \geq 1,
\end{align*}
and,
\begin{align*}
    \|\Tilde{\Theta}-\Bar{\Theta}_{t}\| < \|\Bar{\Theta}_{t}-\Bar{\Theta}_{t+1}\|,
\end{align*}
in a contradiction to the fact that $\Bar{\Theta}_{t+1}$ is the projection of $\Bar{\Theta}_{t+1}$ on $\fset_{\tau_{t}}$.

Since $\Bar{\Theta}_2$ remains zero during all iterations, we identify each $\Bar{\Theta}$ with a vector $(x,y,z)\in \R^3$, where 
$\vecind{\Bar{\Theta}_{t}}{1}=x, \vecind{\Bar{\Theta}_{t}}{3}=y, \vecind{\Bar{\Theta}_{t}}{4}=z$ and analyze the projection sequence in this $3$-dimensional space, where
\[
\fset_1:=\Bigl\{(x,y,z)\in\mathbb R^3:\ a_1 x^2 + b y^2 - b z^2 \ge 1\Bigr\},
\fset_2:=\Bigl\{(x,y,z)\in\mathbb R^3:\ a_2 x^2 - c y^2 + c z^2 \ge 1\Bigr\}.
\]
This will not change the dynamics of the projections.

For simplicity, we assume that $\tau_1=1$, the case where $\tau_1=2$ is analogous.
Under this choice, we prove that $
d\bigl(\Bar{\Theta}_{t},\interset\bigr) \ge 2.99$.
For the first step,
we minimize
\[
x^2+y^2+z^2\quad\text{s.t.}\quad a_1x^2+by^2-bz^2\ge1.
\]
At any minimizer one must have $z=0$: setting $z=0$ strictly decreases the objective and increases feasibility.
Thus we reduce to
\[
\min_{x,y} x^2+y^2\quad\text{s.t.}\quad a_1x^2+by^2\ge1.
\]
Since $b>a_1>0$, by \cref{lem:proj_ellipse_yaxis}, the minimum is attained at $x=0$ and $y^2=1/b$, giving
\[
\Bar{\Theta}_{1}\in A:=\{(0,\pm b^{-1/2},0)\}.
\]
For the next projection, which is onto $\fset_2$, by symmetry it suffices to project $\Bar{\Theta}=(0,b^{-1/2},0)$ (the second case is analogous).
We minimize
\[
F(x,y,z)=x^2+(y-b^{-1/2})^2+z^2
\quad\text{s.t.}\quad
a_2x^2-cy^2+cz^2\ge1.
\]
By \cref{lem:proj_K2_explicit}, it follows that
\[
\Bar{\Theta}_{2}\in
B:=\left\{\left(0,\ \pm\frac{1}{2\sqrt b},\ \pm\sqrt{\frac{1}{4b}+\frac1c}\right)\right\}.
\]
In particular, every $\Bar{\Theta}_{2}=(0,y_0,z_0)\in B$ satisfies
\[
z_0^2-y_0^2=\frac1c\quad\Longrightarrow\quad z_0^2\ge y_0^2+\frac1c.
\]

We continue that for all $t\ge2$, $\vecind{\Bar{\Theta}_{t}}{1}=0$ by induction in $t$.
Assume $\Bar{\Theta}_{t}=(0,y_0,z_0)$
If $\Bar{\Theta}_{t}\in\fset_2$ and the projection is onto $\fset_1$,
then, $a_2\cdot 0^2-cy_0^2+cz_0^2\ge1$, i.e.
$z_0^2\ge y_0^2+\frac1c$.
Consider projecting $\Bar{\Theta}_{t}$ onto $\fset_1$, i.e.
\[
\min_{(x,y,z)}\ \|(x,y,z)-\Bar{\Theta}_{t}\|^2
\quad\text{s.t.}\quad a_1x^2+by^2-bz^2\ge1.
\]
In \cref{lem:kkt_two_branches_detailed}, we show that any minimizer of this optimization problem satisfies $x=0$.

The case where $\Bar{\Theta}_{t}\in\fset_1$ and the projection is onto $\fset_2$, is analogous and is proved in \cref{lem:kkt_projection_f2}.

As a result, we proved that for every $t$, $\vecind{\Bar{\Theta}_{t}}{1}=0$.
It is left to bound from below the distance between $\Bar{\Theta}_{t}$ and the intersection $\interset$.

Let $\Bar{\Theta}\in\interset$. 
Multiply the $\fset_1$ constraint by $c$ and the $\fset_2$ constraint by $b$ and add:
\[
c(a_1\vecind{\Bar{\Theta}}{1}^2-a_1\vecind{\Bar{\Theta}}{2}^2+b\vecind{\Bar{\Theta}}{3}^2-b\vecind{\Bar{\Theta}}{4}^2)\ +\ b(a_2\vecind{\Bar{\Theta}}{1}^2-a_2\vecind{\Bar{\Theta}}{2}^2-c\vecind{\Bar{\Theta}}{3}^2+c\vecind{\Bar{\Theta}}{4}^2)\ \ge\ c+b.
\]
This gives,
\[
(\frac{2}{9}+2\varepsilon^2)\vecind{\Bar{\Theta}}{1}^2=
(ca_1+ba_2)\vecind{\Bar{\Theta}}{1}^2\geq (ca_1+ba_2)\vecind{\Bar{\Theta}}{1}^2-(ca_1+ba_2)\vecind{\Bar{\Theta}}{2}^2\ \ge\ b+c=2.
\]
Thus,
\[
\vecind{\Bar{\Theta}}{1}^2\ \ge\ \frac{2}{\tfrac{2}{9}+2\varepsilon^2}
=\frac{1}{\tfrac19+\varepsilon^2}.
\]
But for all $t\ge0$ we proved $\vecind{\Bar{\Theta}_{t}}{1}=0$, so for any $\Bar{\Theta}\in\interset$,
\[
\|\Bar{\Theta}_{t}-\Bar{\Theta}\|^2\ \ge\ \vecind{\Bar{\Theta}}{1}^2\ \ge\ \frac{1}{\tfrac19+\varepsilon^2}.
\]
Therefore
\[
d(\Bar{\Theta}_{t},\interset)\ \ge\ \sqrt{\frac{1}{\tfrac19+\varepsilon^2}}
>2
\]
as claimed.
For the forgetting,
we notice that for every $t$, there exists $m\in{1,2}$ such that $\Bar{\Theta}_{t}\in \fset_m$.

If $m=1$, since $\Bar{\Theta}_{t}$ is always on the boundary, it holds that
\begin{align*}
    1=y_{1}f(\x_{1};\Bar{\Theta}_{t})=
    \langle u^2,\x_1\rangle -\langle v^2,\x_1\rangle
    = (1-\varepsilon)\vecind{u}{2}^2 -(1-\epsilon)\vecind{v}{2}^2.
\end{align*}
This implies, $\vecind{v}{2}^2=-\frac{1}{1-\epsilon} +\vecind{u}{2}^2$.
Then, for the other task, $m=2$, it holds that,
\begin{align*}
    y_{2}f(\x_{2};\Bar{\Theta}_{t})&=
    \langle u^2,\x_2\rangle -\langle v^2,\x_2\rangle
    \\&=(-1-\varepsilon)\vecind{u}{2}^2 -(-1-\epsilon)\vecind{v}{2}^2
    \\&= (-1-\varepsilon)\vecind{u}{2}^2 + (1+\varepsilon) (-\frac{1}{1-\epsilon} +\vecind{u}{2}^2)
    \\&=-\frac{1+\epsilon}{1-\epsilon}
    \\&\leq -1,
\end{align*}
and, $1-y_{2}f(\x_{2};\Bar{\Theta}_{t})\geq 2$.
This implies,
\begin{align*}
   F_2(\Bar\Theta_t) 
   &\geq (1-y_{2}f(\x_{2};\Bar{\Theta}_{t}))
    \geq 2.
\end{align*} 

If $m=2$, similarly, it holds that
\begin{align*}
    1=y_{2}f(\x_{2};\Bar{\Theta}_{t})=
    \langle u^2,\x_2\rangle -\langle v^2,\x_2\rangle
    = (-1-\varepsilon)\vecind{u}{2}^2 -(-1-\epsilon)\vecind{v}{2}^2.
\end{align*}
This implies, $\vecind{u}{2}^2=-\frac{1}{1+\epsilon} +\vecind{v}{2}^2$.
Then, for the other task, $m=1$, it holds that,
\begin{align*}
    y_{1}f(\x_{1};\Bar{\Theta}_{t})&=
    \langle u^2,\x_1\rangle -\langle v^2,\x_1\rangle
    \\&=(1-\varepsilon)\vecind{u}{2}^2 -(1-\epsilon)\vecind{v}{2}^2
    \\&= -(1-\varepsilon)\vecind{v}{2}^2 + (1-\varepsilon) (-\frac{1}{1+\epsilon} +\vecind{v}{2}^2)
    \\&=-\frac{1-\epsilon}{1+\epsilon}
    \\&\leq -0.9,
\end{align*}
and, $1-y_{1}f(\x_{1};\Bar{\Theta}_{t})\geq 1.9$. then,
\begin{align*}
   F_1(\Bar \Theta_t) 
   &\geq (1-y_{1}f(\x_{1};\Bar{\Theta}_{t}))
    \geq 1.5.
\end{align*} 
\end{proof}

%% file: 95_appendix_local.tex
\section{Proofs for \cref{sec:local}}
\label{app:proofs_local}
In this section, we prove the lemmas used for the proof of \cref{thm:local_convergence}.

\subsection{Additional Regularity Conditions}
\label{app:reg_cond}
We begin with defining more  regularity condition that will be used to prove the required conditions and proving that several conditions imply other conditions.
We begin with the following definition of MFCQ condition.

\begin{definition}[MFCQ condition](e.g. \cite{lyu2019gradient}]
\label{def:mfcq}
Given $m$ functions $h_j:\R^n\to \R$ and
a Set $B=\{h_j\geq 0, \forall j\in[m]\}\subseteq \R^n$, we say the system defining set $B$ satisfies the \textit{MFCQ} at a feasible point $\bar{\x} \in B$ if there exists a vector $\vd \in \mathbb{R}^n$ such that for all active constraints $j \in I_B(\x) = \{ j \mid h_j(\bar \x) = 0 \}$,
\[
\langle \nabla h_j(\bar{\x}), \vd \rangle >0.
\]
In addition, if for all active constraints, $\langle \nabla h_j(\bar{\x}), \vd \rangle >\gamma >0$, we say that B satisfies MFCQ with margin $\gamma$.
\end{definition}

Then, we define another regularity condition, named metric regularity and show that it is implied by MFCQ (\cref{def:mfcq}).

\begin{definition}[Metric regularity, Example 9.44 of \cite{rockafellar1998variational}]
\label{def:robinson_regular}
Let $F:\mathbb R^n\to\mathbb R^m$ be continuously differentiable, let
$D\subseteq\mathbb R^m$ be a closed convex set, and let
$\bar \x \in F^{-1}(D)$.
In addition,
\[ N_D(F(\bar{\x}))= \left\{ \vv \in \mathbb{R}^n \mid \langle \vv, \y - F(\bar \x) \rangle \le 0, \quad \forall \y \in D \right\}.\]
The constraint system
$
\{\vx\mid F(\vx)\in D\}
$
is said to be metrically regular
at $\bar \x$ if
\begin{equation}
\label{eq:met_regular_def}
\Bigl[\ \bm{\lambda}\in N_D(F(\bar \x)) \ \text{and} \
\nabla F(\bar \x)^\top \bm{\lambda} = 0 \ \Bigr]
\;\Longrightarrow\; \bm{\lambda} = 0.
\end{equation}
\end{definition}
\begin{lemma}[MFCQ with margin implies Metric Regularity]
\label{lem:mfcq_regularity_vector}
Let 
$\fset := \{\x\in\mathbb R^n \mid h(\x) \ge 0\}$,
where $h:\mathbb R^n\to\mathbb R^{k}$ is a smooth vector-valued function $h(\x) = (h_{1}(\x), \dots, h_{k}(\x))$.
Assume that $\fset$ satisfies the MFCQ with margin $\gamma$. Then, it satisfies also metric regularity.
\end{lemma}
\begin{proof}
Let $D$ be the $k$-dimensional positive orthant.
Suppose $\bm{\lambda}$ satisfies the condition in \eqref{eq:met_regular_def} for $D$ and $h$. 
Let $d$ the vector from \cref{def:mfcq}.
Then, $\vecind{\bm{\lambda}}{j} \le 0$ for all $j$ and
\[
0 = \vd^\top \nabla h(\bar \x)^\top \bm{\lambda} = \sum_{j \in [k]} \vecind{\bm{\lambda}}{j} \nabla h_{j}(\bar \x)^\top \vd.
\]
Using the MFCQ condition ($\nabla h_{j}(\bar \x)^\top \vd \ge \gamma$) and since $\vecind{\bm{\lambda}}{j} \le 0$ (by the fact that $\bm{\lambda}\in N_D (h(\bar \x))$, it holds that
\[
0 \le \gamma \sum_{j \in [k]} \vecind{\bm{\lambda}}{j}\leq 0.
\]
Then, it follows that $\bm{\lambda}=0$.
\end{proof}

Now, we prove that MFCQ and metric regularity of each set implies $\kappa$-linear regularity with respect to the intersection of the sets.

For the proof we use the following claim from \citet{rockafellar1998variational}.

\begin{proposition}[Property of metric regularity]
(Example 9.44 in \cite{rockafellar1998variational})
\label{lem:944}
   Let $F: \mathbb{R}^n \to \mathbb{R}^m$ be a  $\beta$ smooth mapping and let $D \subseteq \mathbb{R}^m$ be a closed set. Define the constraint set $\fset = F^{-1}(D) = \{\x \mid F(\vx) \in D\}$.
Let $\bar{\x} \in \fset$. 
If the constraint set satisfies metric regularity at $\bar{\x}$ and MFCQ with margin $\gamma$ at $\bar{\x}$.
Then, for $\mu =\max_{\y\in N_D(F(\bar \x)),\|\y\|=1}\frac{1}{\nabla F(\bar \x)\y}$, for any $\x$ near $\bar \x$ in holds that
\begin{equation*} 
\dist(\x,\fset)\le \mu \dist(F(\vx), D)
\end{equation*}
\end{proposition}

\begin{lemma}[Metric regularity implies Linear Regularity]
\label{lem:rcq_implies_linear_reg}
Let $M$. For each $m\in M$, let the set $\fset_m$ be defined by a system of inequalities, $
\fset_m := \{\x\in\mathbb R^n \mid h_m(\x) \ge 0\}$,
where $h_m:\mathbb R^n\to\mathbb R^{n_m}$ is a smooth vector-valued function $h_m(\x) = (h_{m,1}(\x), \dots, h_{m,n_m}(\x))$.
Define the global feasible set as the intersection:
\(
\interset := \bigcap_{m\in M}\fset_m.
\)
Fix $\bar{\x}\in\mathcal \interset$.
If every $h_{m,j}$ is $G$-Lipschitz and every $\fset_m$ satisfies metric regularity at $\bar \x$. 
Then the collection $\{\fset_m\}_{m\in M}$ is $\kappa$ linearly regular at $\bar \x$ with $\kappa=\frac{G\|\vv\|}{\gamma}$.
\end{lemma}

\begin{proof} For any unit vector $\bm{\lambda} \in N_D(h_m(\bar \x))$, where $D$ is the $n_m$-dimensional positive orthant, we have $\vecind{\bm{\lambda}}{j} \leq 0$. Then, by the fact that $\|\vu\|=\sup_{\w\neq 0}\frac{\langle \vv,\w\rangle}{\|\w\|}$, and MFCQ with margin $\gamma$ and vector $v$,
\[
\|\nabla h_m(\bar \x)^\top \bm{\lambda}\| \ge \frac{|\langle \vv, \nabla h_m(\bar \x)^\top \bm{\lambda} \rangle|}{\|\vv\|} = \frac{|\sum_j \vecind{\bm{\lambda}}{j} \langle \nabla h_{m,j}(\bar \x), \vv \rangle|}{\|\vv\|} \ge \frac{\gamma \sum |\vecind{\bm{\lambda}}{j}|}{\|\vv\|} \ge \frac{\gamma}{\|\vv\|}.
\]
Thus, by \cref{lem:944}, $\mu = \sup \frac{1}{\|\nabla h_m(\bar \x)^\top \bm{\lambda}\|} \leq \frac{\|\vv\|}{\gamma}$. 
Now let $h$ be the function which is components are all of the constraints $\{h_{m,i}\}$.
Than, since $\interset=h^{-1}(D)$, by \cref{lem:944}, it holds that
\begin{equation} \label{eq:metric_reg_vector}
\dist(\x,\interset)\le \mu \dist(h(\vx), D) = \mu \max_{(m,j)} \max\{0, -h_{m,j}(\x)\}= \frac{\|\vv\|}{\gamma}\max_{(m,j)} \max\{0, -h_{t,j}(\x)\}.
\end{equation}
Now, it holds that,
\[
\max_{(m,j)} \max\{0, -h_{m,j}(\x)\} = \max_{m \in M} \left( \max_{j=1\dots n_m} \max\{0, -h_{m,j}(\x)\} \right), 
\]
thus, for $a_m(\x) := \max_{j} \max\{0, -h_{m,j}(\x)\}$, we get that,
$$
\dist(\x,\interset) \le \frac{\|\vv\|}{\gamma} \max_{m \in M} a_m(\x).
$$
Now, Let $m \in M$ and $\x$. If $\x \in \fset_m$, then $a_m(\x)=0$. Otherwise, let $\vp = \proj_{\fset_m}(\x)$.
By definition, $\vp \in \fset_m$, so $h_{m,j}(\vp) \ge 0$ for all $j=1\dots n_m$.
For any specific component $j$:
\[
-h_{m,j}(\x) \le h_{m,j}(\vp) - h_{m,j}(\x) \le G \|\vp - \x\| = G \dist(\x, \fset_m)
.
\]
Since this holds for every $j$, it holds for the maximum:
\[
a_m(\x) = \max_j \max\{0, -h_{m,j}(\x)\} \le G \dist(\x, \fset_m).
\]
Combining all together, we get that, 
\[
\dist(\x,\interset) \le  \frac{\|\vv\|}{\gamma}\max_{m} \left( G \dist(\x, \fset_m) \right) = \frac{G\|\vv\|}{\gamma} \max_{m} \dist(\x, \fset_m).
\]
\end{proof}

Now we prove that metric regularity implies $(\epsilon,\delta)$ regularity.
In fact, we prove that for any $\epsilon>0$ there exists a $\delta>0$ such that $(\epsilon,\delta)$ regularity is implied.
In the proof, we use the following claim from \cite{rockafellar1998variational}.

\begin{proposition}[Theorem 6.31 in \cite{rockafellar1998variational}]\linebreak
\label{lem:rw_614}
Let $F: \mathbb{R}^n \to \mathbb{R}^m$ be a smooth vector valued mapping and let $D \subseteq \mathbb{R}^m$ be a closed set. 
Define for every set $A$\[
N^{\mathrm{prox}}_{A}(\vu)
=
\left\{
\vv \in \mathbb{R}^{n}
\;\middle|\;
\exists\, \sigma > 0 \text{ s.t. } \vu \in \operatorname{\proj}_{A}(\vu + \sigma \vv).
\right\}
\]
Define the constraint set $\fset = F^{-1}(D) = \{\x \mid F(\vx) \in D\}$.
Then for any $\bar{\x} \in \fset$, if metric regularity holds at $\bar{\x}$, and $D$ is convex, then,
\begin{equation}
N_\fset^{\mathrm{prox}}(\bar{\x}) =\{ \nabla F(\bar{\x})^\top \bm{\lambda} \mid \bm{\lambda} \in N_D^{\mathrm{prox}}(F(\bar{\x}))\}.
\end{equation}
\end{proposition}

\begin{lemma}[Metric regularity implies $(\varepsilon, \delta)$-Regularity]
\label{lem:rcq_implies_super_reg}
Let $F: \mathbb{R}^n \to \mathbb{R}^m$ be a smooth vector valued mapping, and let $D \subseteq \mathbb{R}^m$ be a closed convex set.
Define the feasible set $\fset = \{ \x \in \mathbb{R}^n \mid F(\vx) \in D \}$.
Assume that metric regularity (\cref{def:robinson_regular}) and MFCQ holds at $\bar \x \in \fset$ with margin $\gamma$.
Then, for every $\varepsilon > 0$, $\fset$ is $\left((\varepsilon,\frac{\epsilon r}{\beta \|\bar \x\|}\right)$-regular at $\bar \x$.
\end{lemma}

\begin{proof}
Let $\varepsilon > 0$ be given. We need to find $\delta > 0$ such that for all $\x, \y \in \fset \cap \ball_\delta(\bar \x)$ and any proximal normal $\vv \in N_\fset^{\mathrm{prox}}(\x)$, the inequality $\langle \vv, \y - \x \rangle \le \varepsilon \|\vv\| \|\y - \vx\|$ holds.

By \cref{lem:rw_614}, 
it holds that for $\x$ near $\bar \x$, any proximal normal $v \in N_\fset^{\mathrm{prox}}(\x)$ can be represented as:
\[
\vv = \nabla F(\vx)^\top \bm{\lambda} \quad \text{with} \quad \bm{\lambda} \in N_D^{\text{prox}}(F(\vx)).
\]
Since $F$ is smooth vector valued mapping we have $\beta'$ such that,
\[
\|F(\y) - F(\vx) - \nabla F(\vx)(\y - \vx)  \|\leq \frac{\beta'}{2}\|\y - \vx\|^2
\]
By \cref{lem:rw_614}, and using the Cauchy-Schwarz inequality on the error term, it holds that:
\begin{align*}
    \langle \vv, \y - \vx \rangle 
    &= \langle \nabla F(\vx)^\top \bm{\lambda}, \y - \vx \rangle \\
    &= \langle \bm{\lambda}, \nabla F(\vx)(\y - \vx) \rangle \\
    &= \langle \bm{\lambda}, F(\y) - F(\vx) - \left( F(\y) - F(\vx) - \nabla F(\vx)(\y - \vx) \right) \rangle \\
    &= \langle \bm{\lambda}, F(\y) - F(\vx) \rangle - \langle \bm{\lambda}, F(\y) - F(\vx) - \nabla F(\vx)(\y - \vx) \rangle \\
    &\leq \underbrace{\langle \bm{\lambda}, F(\y) - F(\vx) \rangle}_{\leq 0} + \|\bm{\lambda}\| \left\| F(\y) - F(\vx) - \nabla F(\vx)(\y - \vx) \right\| \\
    &\leq  \frac{\beta'
    \|\bm{\lambda}\|}{2} \|\y - \vx\|^2.
\end{align*}
The first term is $\le 0$ because 
$D$ is convex (and thus, $N_D((F(\vx))=N_D^{\text{prox}}(F(\vx))$), $\bm{\lambda} \in N_D(F(\vx))$, and $F(\y) \in D$.
For the second term,by  MFCQ with margin $\gamma$ and vector $\vu=\bar \x$,
\[
\|\vv\|=\|\nabla F(\bar \x)^\top \bm{\lambda}\| \ge \frac{|\langle \vu, \nabla F(\bar \x)^\top \bm{\lambda} \rangle|}{\|\vu\|} = \frac{|\sum_j \vecind{\bm{\lambda}}{j} \langle \nabla F_j(\bar \x), \vu \rangle|}{\|\vu\|} \ge \frac{\gamma \sum |\vecind{\bm{\lambda}}{j}|}{\|\vu\|}  \ge \|\bm{\lambda}\|\frac{\gamma}{\|\vu\|}.
\]
This implies, 
$$
\|\bm{\lambda}\| \le  \frac{\|\vv\|\|\vu\|}{\gamma}.
$$

Thus, combining all together, we get,
\[
\langle \vv, \y - \vx \rangle \le \frac{\|\vv\|\|\vu\|}{\gamma} \left(\frac{\beta'}{2} \|\y - \vx\|^2\right).
\]
Thus,
to satisfy the $(\varepsilon, \delta)$-regularity condition, we denote $\kappa= \frac{\|\vu\|}{\gamma}=\frac{\|\bar \x\|}{\gamma}$. We need to choose $\delta$ such that for every $\x, \y \in \ball_\delta(\bar \x)$
\[
\frac{\kappa \beta'}{2} \|\vv\| \|\y - \vx\|^2 \le \varepsilon \|\vv\| \|\y - \vx\|.
\]
Dividing by $\|\vv\| \|\y-\x\|$ (assuming nonzero, otherwise the inequality holds trivially), this requires $\frac{\kappa \beta'}{2} \|\y - \vx\| \le \varepsilon$.
Since $\x, \y \in \ball_\delta(\bar \x)$, we have $\|\y - \vx\| \le 2\delta$.
Thus, it suffices to choose $\delta$ such that:
\[
\frac{\kappa \beta'}{2} (2\delta) \le \varepsilon \implies \delta \le \frac{\varepsilon}{\kappa \beta'}.
\]
With this $\delta$, the set is $(\varepsilon, \delta)$-regular.
\end{proof}

\subsection{Proof of \cref{thm:local_convergence}}

In this section we prove \cref{thm:local_convergence}.
We begin with the proof of \cref{prop:one_step_random_projection}.

\begin{proof} [of \cref{prop:one_step_random_projection}]
The statement for cyclic ordering follows directly from Corollary 5.10 in \cite{dao2019linear}.
Here we prove the statement for random ordering.
Fix $\x_{t-1}\in \ball_{\delta/2}(\opt)$ and let $\bar{\x}\in\proj_{\interset}(\x_{t-1})$. 
Since $\opt \in \interset$, we have 
\[
\|\x_{t-1}-\bar{\x}\| 
= \dist(\x_{t-1},\interset) 
\le \dist(\x_{t-1},\opt) \le \delta/2
\,.
\]
By the triangle inequality,
\[
\dist(\bar{\x},\opt) 
\le 
\norm{\x_{t-1}-\bar{\x}}
+
\dist(\x_{t-1},\opt) 
\le \frac{\delta}{2} + \frac{\delta}{2} = \delta
\,.
\]
Thus, both $\x_{t-1}$ and $\bar{\x}$ lie in $\ball_{\delta}(\opt)$, ensuring the local regularity assumptions apply.

Let $\tau(t)\in\{1,\dots,M\}$ be arbitrary and let $\x_{t} \in \proj_{\fset_{\tau(t)}}(\x_{t-1})$.
By $(\varepsilon,\delta)$--regularity and Proposition~3.5 of \citet{dao2019linear}
(with $\bm{\lambda}=1$), the exact projector $\proj_{\fset_{\tau(t)}}$ is
$(\Omega_{\tau(t)},\gamma,\beta)$--quasi firmly Fej\'er monotone on
$\ball_{\delta/2}(\opt)$ with
\[
\gamma = \frac{1}{1-\varepsilon},
\qquad
\beta = 1,
\qquad
\Omega_{\tau(t)} \triangleq \fset_{\tau(t)} \cap \ball_{\delta}(\opt)
\,.
\]
Since $\bar{\x}\in \interset \subseteq \fset_{\tau(t)}$ and
$\bar{\x}\in \ball_{\delta}(\opt)$, we have
$\bar{\x}\in \Omega_{\tau(t)}$, and hence the quasi firm Fej\'er
inequality yields:
\[
\|\x_{t} - \bar{\x}\|^2 + \|\x_{t-1}-\x_{t}\|^2
\;\le\;
\gamma \|\x_{t-1}-\bar{\x}\|^2 
=
\gamma \dist^2 (\x_{t-1},\interset)
\,.
\]
Using $\dist(\x_{t}, \interset)
\le \|\x_{t} - \bar{\x}\|$
and
$\|\x_{t-1}-\x_{t}\| = \dist(\x_{t-1},\fset_{\tau(t)})$ (exact projection), 
we obtain
\[
\dist^2(\x_{t},\interset)
\;\le\;
\gamma \dist^2(\x_{t-1},\interset) - \dist^2(\x_{t-1},\fset_{\tau(t)}) 
\,.
\]

Taking expectation with respect to $\tau(t)$ yields
\[
\mathbb{E}_{\tau(t)}[\dist^2(\x_{t},\interset)]
\;\le\;
\gamma \dist^2(\x_{t-1},\interset)
-
\mathbb E_{\tau(t)}[\dist^2(\x_{t-1},\fset_{\tau(t)})]\,.
\]
By linear regularity,
\[
\max_{1\le m\le M} \dist(\x_{t-1},\fset_{m}) 
\ge 
\frac{1}{\kappa} \dist(\x_{t-1},\interset)\,,
\]
and therefore
\[
\mathbb E_{\tau(t)}[\dist^2(\x_{t-1},\fset_{\tau(t)})]
=
\frac{1}{M}\sum_{m=1}^M \dist^2 (\x_{t-1},\fset_{m})
\;\ge\;
\frac{1}{M\kappa^2} \dist^2(\x_{t-1},\interset)
\,.
\]
Overall, we showed,
\[
\mathbb{E}_{\tau(t)}[\dist^2(\x_{t},\interset)]
\;\le\; 
\left(
\gamma
-
\frac{1}{M\kappa^2}
\right)
\dist^2(\x_{t-1},\interset)
\,.
\]
Substituting $\gamma=\frac{1}{1-\varepsilon}$ completes the proof.
\end{proof}

\begin{lemma}
\label{lem:loss_bound}
Let $f(\x;\Theta)$ and $\fset_1,\ldots\fset_M;\interset$ as in \cref{sec:setup}.
Let $\delta$ and $\Theta$ be such that $\dist(\Theta,\interset)\leq \delta$.
Let $\opt=\proj_{\interset}(\Theta)$,
and assume that $f(\x,\cdot)$ is $G$-Lipschitz in $\ball_\delta(\opt)=\{\Theta \in \R^\paramdim \mid \|\Theta-\opt\|\leq \delta \}$.
Then, it holds that, 
\[\max_{m} F_m(\Theta)\leq G \dist (\Theta,\interset)\]

\begin{proof}
Let $m\in[M]$ and $(\x_i^{(m)},y_i^{(m)})$. By the fact that for every such data point, it holds that $y_i^{(m)} f(\x_i^{(m)};\opt)\geq 1$ and $G$-Lipschitzness, it holds that,
\[
1-y_i^{(m)}f(\x_i^{(m)};\Theta)\leq 
y_i^{(m)} f(\x_i^{(m)};\opt)- y_i^{(m)} f(\x_i^{(m)};\Theta)
\le G|y| \|\opt-\Theta\|
= G \dist(\Theta,\interset)
\]
Since this holds for any data point, it holds also for the maximal, thus, 
\[
\max_{m}F_m(\Theta)\triangleq \max_{i,m}\max\{0,\,1-y_i^{(m)} f(\x_i^{(m)};\Theta)\}
\le G\dist(\Theta,\interset).
\]
\end{proof}
\end{lemma}

Now,
we use the lemmas form \cref{app:reg_cond} to show that if for every $\x$, $f(\x;\Theta)$ is $G$-Lipschitz, smooth and positively homogeneous functions to show that the regularity conditions mentioned in \cref{prop:one_step_random_projection} hold.

We begin with the following lemma that shows that for homogeneous models with degree $\homogendeg$, each feasible set satisfies the MFCQ condition at $\opt$ given in \cref{def:mfcq} with margin $\homogendeg$ such the vector $d$ in this definition is $\opt.$

\begin{lemma}[MFCQ of homogeneous models]
\linebreak
\label{lem:mfcq}
Consider a \(\homogendeg\)-positively-homogeneous model $f(\cdot;\Bar \Theta):\mathcal X\to\mathbb R$ that is $G$-Lipschitz and $\beta$-$smooth$.
Assume joint separability, i.e., nonempty intersection $\interset=\fset_{1}\cap \dots\cap\fset_{M}\neq \emptyset$.
Let $\opt \in \interset=\bigcap_m\fset_m$.
    Then, each $\fset_m$ satisfies the MFCQ condition at $\opt$ with margin $\homogendeg$ for $d=\opt$.
In addition, also $\interset$ satisfies the MFCQ condition at $\opt$ with margin $\homogendeg$ for $\vd=\opt$.
\end{lemma}
\begin{proof}
Let $m$
and let $I_m:=\{i: y_i^{(m)} f(\opt,\x_i^{(m)})=1\}=\{i: y_i^{(m)} f(\opt,\x_i^{(m)})=1\}$.
If $I_m$ is empty, the definition holds trivially. Otherwise,
since $f(x;\Theta)$ is smooth and positively homogeneous of degree $\homogendeg$, Euler's theorem for homogeneous functions yield,  for every $\x$, $
\langle \nabla_\Theta f(\x; \opt),\opt\rangle=\homogendeg f(\x; \opt).
$
Then, for every $i\in I_m$, $h_{m,i}=y_i^{(m)}f(\opt,\x_i^{(m)})-1=0$,
satisfies,
\[
\langle \nabla h_{m,i}(\opt),\opt\rangle
= y_i^{(m)} \langle \nabla_\Theta f(\opt,\x_i^{(m)}),\opt\rangle
= y_i^{(m)} \homogendeg f(\opt,\x_i^{(m)})=\homogendeg>0.
\]
Then, the MFCQ condition holds for $\vd=\opt$ and margin $\homogendeg$.
The proof for $\interset$ is identical.
\end{proof}

Now we turn to prove $(\epsilon,\delta)$ regularity.

\begin{lemma}[($\epsilon,\delta$)
regularity of homogeneous models]
\label{lem:epsdel_reg}
Consider a \(\homogendeg\)-positively-homogeneous model $f(\cdot;\Bar \Theta):\mathcal X\to\mathbb R$ that is $G$-Lipschitz and $\beta$-$smooth$.
Assume joint separability, i.e., nonempty intersection $\interset=\fset_{1}\cap \dots\cap\fset_{M}\neq \emptyset$.
Let $\opt \in \interset=\bigcap_t\fset_m$.
Then, for every $\epsilon>0$,
there exists $\delta>0$ such that 
every set $\fset_m$ is ($\epsilon,\delta$)-regular at $\opt$.
\end{lemma}
\begin{proof}
Let $m\in M$
    By \cref{lem:mfcq} MFCQ condition holds with margin $\homogendeg$. In addition, by \cref{lem:mfcq_regularity_vector} metric regularity also holds. Then, by \cref{lem:rcq_implies_super_reg}, the lemma follows. 
\end{proof}

\begin{lemma}[$\kappa$-linear regularity of homogeneous models]\linebreak
\label{lem:kap_reg}
Consider a \(\homogendeg\)-positively-homogeneous model $f(\cdot;\Bar \Theta):\mathcal X\to\mathbb R$ that is $G$-Lipschitz and $\beta$-$smooth$.
Assume joint separability, i.e., nonempty intersection $\interset=\fset_{1}\cap \dots\cap\fset_{M}\neq \emptyset$.
Let $\opt \in \interset=\bigcap_m\fset_m$.
Then, if $f$ is G-Lipschitz, the collection $\{\fset_m\}_{m=1}^M$ satisfy $\kappa$-linear regularity for $\kappa=\frac{G\|\opt\|}{\homogendeg}$ around $\opt$.
\end{lemma}
\begin{proof}
Let $m\in M$
    By \cref{lem:mfcq}, MFCQ condition holds for the vector $\|\opt\|$ with margin $\homogendeg$. In addition, by \cref{lem:mfcq_regularity_vector} metric regularity also holds. Then, by \cref{lem:rcq_implies_linear_reg}, the lemma follows. 
\end{proof}

Now we can turn to the proof of \cref{thm:local_convergence}.

\begin{proof}[of \cref{thm:local_convergence}]
For the cyclic ordering,
let $\epsilon$ be such that $
\frac{1}{(1-\varepsilon)^{M-1}} = 1 + \frac{1}{2(M-1)\kappa^2}$.
Then by \cref{lem:epsdel_reg} each set $\fset_m$ is $(\epsilon,\delta)$ regular.
In addition, by \cref{lem:kap_reg}, the collection $\{\fset_m\}_{m=1}^M$ satisfy $\kappa$-linear regularity for $\kappa=\frac{G\|\opt\|}{\homogendeg}$ in a around $\opt$.
As a result, the conditions of \cref{prop:one_step_random_projection} holds in a $\delta$-neighborhood of $\opt$. By the choice of $\epsilon$
we get, for the cyclic order, that for any 
$t$, $\dist(\Bar{\Theta}_{t+M},\interset)
\;\le\;
\rho\;\dist(\Bar{\Theta}_{t},\interset)$
\begin{align*}
\rho 
&= 
\left( 1 - \frac{1}{2(M-1)\kappa^2} \right)^{\hfrac{1}{2(M-1)}} \\
 &= 
 \left( 1 - \frac{1}{2(M-1)(G\|\opt\|/\homogendeg)^2} \right)^{\hfrac{1}{2(M-1)}} 
 \\
 &= 
 \left( 1 - \frac{\homogendeg^2}{2(M-1)G^2\|\opt\|^2} \right)^{\hfrac{1}{2(k-1)}}
 \\&
 \le \left( \exp \left( - \frac{\homogendeg^2}{2MG^2 \|\Theta^\star\|^2} \right) \right)^{\hfrac{1}{2M}} 
 \\
&= 
\exp \left( - \frac{\homogendeg^2}{4M^2 G^2 \|\Theta^\star\|^2} \right)
\,,
\end{align*}
Thus, for $M\mid k$
\begin{align*}
\dist(\Bar \Theta_{k}, \interset)
\leq
\exp \left(
-
\frac{k\homogendeg^2}{4MG^2\|\opt\|^2}
\right)
\dist(\Bar \Theta_{0},\interset)
\end{align*}
For the random order, using \cref{prop:one_step_random_projection}, we get that for $\epsilon$ such that such that $
\frac{1}{
1-\varepsilon} = 1 + \frac{1}{2M\kappa^2}$, 
\[
\mathbb E_{\tau(t)}\!\left[\dist^2(\Bar \Theta_{t}, \interset)\right]
\;\le\;
\left(
1
-
\frac{1}{2M\kappa^2}
\right)
\dist^2(\Bar \Theta_{t-1},\interset) .
\]
Thus, for $k=cM$,
\begin{align*}
\mathbb E_\tau\!\left[\dist^2(\Bar \Theta_{k}, \interset)\right]
\;&\le\;
\left(
1
-
\frac{1}{2M\kappa^2}
\right)^k
\dist^2(\Bar \Theta_{0},\interset)
\\&\leq\exp \left(
-
\frac{k}{2M\kappa^2}
\right)
\dist^2(\Bar \Theta_{0},\interset)
\\&=
\exp \left(
-
\frac{k\homogendeg^2}{2MG^2\|\opt\|^2}
\right)
\dist^2(\Bar \Theta_{0},\interset)
\end{align*}
By Jensen inequality, 
\begin{align*}
\mathbb E_{\tau}\!\left[\dist(\Bar \Theta_{k}, \interset)\right]
\;&\le\;
\sqrt{\mathbb E_{\tau}\!\left[\dist^2(\Bar \Theta_{k}, \interset)\right]}
\leq
\exp \left(
-
\frac{k\homogendeg^2}{4MG^2\|\opt\|^2}
\right)
\dist(\Bar \Theta_{0},\interset)
\end{align*}
For the forgetting, the statement follows by \cref{lem:loss_bound}.
\end{proof}

%% file: 97_appendix_regression.tex
\section{Proofs for \cref{sec:regression}}
\label{app:regression}
\subsection{Proof of \cref{thm:gen_m_convergence_reg}}
\begin{proof}[of \cref{thm:gen_m_convergence_reg}]
As in classification, we employ the theory of $\Gamma$-convergence. Let $G_t^{(\lambda)}(\Theta) := \frac{1}{\lambda}\mathcal{L}_{\tau(t)}(\Theta) + \|\Theta - \Theta_{t-1}\|^2$.
It holds that $\Theta \in \argmin_{\Theta \in \fset_{\tau(t)}}\mathcal{L}_{\tau(t)}(\Theta) + \lambda\|\Theta - \Theta_{t-1}\|^2$ if and only if $\Theta \in \argmin_{\Theta \in \fset_{\tau(t)}} G_t^{(\lambda)}(\Theta)$. 
We show that $G_t^{(\lambda)}$ $\Gamma$-converges to the following funcation as  $\lambda\downarrow0$:
\[
G_t(\Theta) = \begin{cases} \|\Theta - \Theta_{t-1}\|^2 & \text{if } \Theta \in \fset_{\tau(t)} \\ +\infty & \text{otherwise.} \end{cases}
\]
First, for the
Liminf inequality, let $\Theta_t^{(\lambda)} \to \Theta$. 
By the definition of the feasible set, if $\Theta \notin \fset_{\tau(t)}$, then $\mathcal{L}_m(\Theta) > 0$. Let $\mathcal{L}_m(\Theta) = \delta > 0$. Since $f$ is continuous, $\mathcal{L}_m$ is continuous. 

Therefore, there exists a neighborhood around $\Theta$ where $\mathcal{L}_m(\cdot) > \delta/2$.
    Since $\Theta_t^{(\lambda)} \to \Theta$, for sufficiently small $\lambda$, $\Theta_t^{(\lambda)}$ lies in this neighborhood. Thus:
    \[
    G_t^{(\lambda)}(\Theta_t^{(\lambda)}) = \frac{1}{\lambda}\mathcal{L}_m(\Theta_t^{(\lambda)}) + \bignorm{\Theta_t^{(\lambda)} - \Theta_{t-1}}^2 \ge \frac{\delta}{2\lambda} + 0.
    \]
    Taking the limit as $\lambda\downarrow0$, we have $\lim G_t^{(\lambda)}(\Theta_t^{(\lambda)}) = +\infty$. Since $G_t(\Theta) = +\infty$, the inequality holds.
    
    Otherwise, if $\Theta \in \fset_{\tau(t)}$.
    Since $\mathcal{L}_m(\cdot) \ge 0$. we have:
    \[
    G_t^{(\lambda)}(\Theta_t^{(\lambda)}) \ge \bignorm{\Theta_t^{(\lambda)} - \Theta_{t-1}}^2.
    \]
    Thus, by continuity, as $\Theta_t^{(\lambda)} \to \Theta$, the term $\|\Theta_t^{(\lambda)} - \Theta_{t-1}\|^2 \to \|\Theta - \Theta_{t-1}\|^2$
    and \linebreak $\liminf_{\lambda\downarrow0} G_t^{(\lambda)}(\Theta_t^{(\lambda)}) \ge \|\Theta - \Theta_{t-1}\|^2 = G_t(\Theta)$.

For the Limsup property, 
We need to show that for any $\Theta$, there exists a sequence $\Theta_t^{(\lambda)} \to \Theta$ such that $\limsup_{\lambda\downarrow0} G_t^{(\lambda)}(\Theta_t^{(\lambda)}) \le G_t(\Theta)$.  
We choose the constant sequence $\Theta_t^{(\lambda)} = \Theta$ for all $\lambda$.
If $\Theta \notin \fset_{\tau(t)}$,
    $G_t(\Theta) = +\infty$, so the inequality $\limsup G_t^{(\lambda)}(\Theta_t^{(\lambda)}) \le \infty$ is trivially satisfied.
    Otherwise, if $\Theta \in \fset_{\tau(t)}$, then $\mathcal{L}_m(\Theta) = 0$ and $G_t^{(\lambda)}(\Theta) = G_t(\Theta)$.

Finally for the boundedness of level sets, the function $G_t^{(\lambda)}(\Theta)$ is bounded below by $\|\Theta - \Theta_{t-1}\|^2$. 
Since the sublevel sets of this lower bound are compact, the sequence of functionals is equi-coercive.

The theorem follows by Part (b) of \cref{thm:fundamental_gamma} and the fact that $\Theta_t$ is a minimizer of $G_t$.
\end{proof}

\newpage

\subsection{Proof of \cref{thm:lower_equality}}

\begin{lemma}[Catastrophic Forgetting in Squared Models]
\label{thm:lower_equality_formal}
Consider a squared model $f(\x;\piterate{})$ in $d=2$ parameterized by $\piterate{}=(\vu,\vv)$ where $\vu,\vv\in\R^{2}$.
Fix $\varepsilon\in(0,0.01)$ and define two tasks
\[
\X^{(1)}=\big[\x_1^{\top}\big]
=\big[\bigl(\tfrac19+\varepsilon,\ 1-\varepsilon\bigr)\big],
\qquad
\X^{(2)}=\big[\x_2^{\top}\big]
=\big[\bigl(\tfrac19-\varepsilon,\ -1-\varepsilon\bigr)\big],
\qquad 
\y^{(1)}=\y^{(2)}=\prn{1}\,.
\]
%
Then,
\begin{itemize}[leftmargin=0.7cm, itemindent=0cm, itemsep=0pt,labelsep=0.2cm,topsep=4pt]
\item[(a)] 
The induced feasible sets $\fset_1=\{\piterate{} \in \R^4 \mid f(\x_1;\piterate{})= y_1\},\, \fset_2=\{\piterate{} \in \R^4 \mid f(\x_2;\piterate{})=y_2\}$ have a nonempty intersection,
\ie 
$\interset:=\fset_1\cap \fset_2 \neq\varnothing$.
\item[(b)]\cref{alg:sequential_mm}
satisfies, for all $t\ge0$, $
d\bigl(\Bar{\Theta}_t,\interset\bigr)
\ \ge\ 2$ and $\max_mF_{m}(\Bar{\Theta}_{t})\geq 3$
.
\end{itemize}
\end{lemma}

\begin{proof}[of \cref{thm:lower_equality_formal}]
We now consider the dynamics where the task sets are defined by the same constraints as \cref{thm:lower}, except there is equality instead of inequality. 
However, 
we observe that the sequential projections onto these equality sets generate the exact same trajectory as the projections onto the inequality sets $\fset_m = \{ \Bar{\Theta} \mid f(x_m; \Bar{\Theta}) \ge 1 \}$.

This follows immediately from \cref{lem:proj_ellipse_yaxis}, \cref{lem:proj_K2_explicit}, and \cref{lem:kkt_two_branches_detailed} in \cref{app:proofs_squared}. Specifically, those lemmas  prove that for every projection step involved in the inequality-constrained dynamics, the unique minimizer satisfies the constraint with strict equality (\ie $f(\x_m; \Bar{\Theta}) = 1$). 
Since the solution to the relaxed problem ($\ge 1$) lies on the boundary, it is necessarily the solution to the equality-constrained problem. Thus, the sequence of iterates $\{\Bar{\Theta}_t\}$ is identical in both settings.
By those lemmas we get that for every $t$ it holds for $\Bar{\Theta}_t=(\vu_t,\vv_t)$ that $\vu_t=0$.

Now, let $\Bar{\Theta}=(u,v)\in\interset$. 
Multiply the $\fset_1$ constraint by $c$ and the $\fset_2$ constraint by $b$ and add:
\begin{align*}
   &c(a_1\vecind{\opt}{1}^2-a_1\vecind{\opt}{2}^2+b\vecind{\opt}{3}^2-b\vecind{\Bar{\Theta}}{4}^2)\ \\&+\ b(a_2\vecind{\opt}{1}^2-a_2\vecind{\opt}{2}^2-c\vecind{\opt}{3}^2+c\vecind{\Bar{\Theta}}{4}^2)\\&=c+b=2. 
\end{align*}

This gives,
\[
(\frac{2}{9}+2\varepsilon^2) (\vecind{\opt}{1}^2-\vecind{\opt}{2}^2)=(ca_1+ba_2)\vecind{\opt}{1}^2-(ca_1+ba_2)\vecind{\opt}{2}^2=2.
\]
In particular,
\[\vecind{\opt}{1}^2\geq 
\vecind{\opt}{1}^2-\vecind{\opt}{2}^2
=\frac{1}{\tfrac19+\varepsilon^2}.
\]
But for all $t\ge0$ we proved $\vecind{\Bar{\Theta}_{t}}{1}=0$, so for any $\Bar{\Theta}\in\interset$,
\[
\|\Bar{\Theta}_{t}-\Bar{\Theta}\|^2\ \ge\ \vecind{\opt}{1}^2\ \ge\ \frac{1}{\tfrac19+\varepsilon^2}.
\]
Therefore,
\[
d(\Bar{\Theta}_{t},\interset)\ \ge\ \sqrt{\frac{1}{\tfrac19+\varepsilon^2}}
>2
\]
as claimed.
For the loss,
we notice that for every $t$, there exists $m\in{1,2}$ such that $\Bar{\Theta}_{t}\in \fset_m$.

If $m=1$, since $\Bar{\Theta}_{t}$ is always on the boundary, it holds that
\begin{align*}
    1=y_{1}=f(\x_{1};\Bar{\Theta}_{t})=
    \langle \vu_t^2,\x_1\rangle -\langle \vv_t^2,\x_1\rangle
    = (1-\varepsilon)\vecind{\vu_t}{2}^2 -(1-\epsilon)\vecind{\vv_t}{2}^2.
\end{align*}
This implies, $\vecind{\vv_t}{2}^2=-\frac{1}{1-\epsilon} +\vecind{\vu_t}{2}^2$.
Then, for the other task, $m=2$, it holds that,
\begin{align*}
    f(\x_{2};\Bar{\Theta}_{t})&=(-1-\varepsilon)\vecind{\vu_t}{2}^2 -(-1-\epsilon)\vecind{\vv_t}{2}^2
    \\&= (-1-\varepsilon)\vecind{\vu_t}{2}^2 + (1+\varepsilon) (-\frac{1}{1-\epsilon} +\vecind{\vu_t}{2}^2)
    \\&=-\frac{1+\epsilon}{1-\epsilon}
    \\&\leq -1\,,
\end{align*}
This implies $F_{2}(\Bar{\Theta}_{t})=|f(\x_{2};\Bar{\Theta}_{t}) - y_2|^2\geq 4$.

If $m=2$, similarly, it holds that
\begin{align*}
    1=y_{2}=f(\x_{2};\Bar{\Theta}_{t})=
    \langle u^2,\x_2\rangle -\langle v^2,\x_2\rangle
    = (-1-\varepsilon)\vecind{\vu}{2}^2 -(-1-\epsilon)\vecind{\vv}{2}^2.
\end{align*}
This implies, $\vecind{\vu}{2}^2=-\frac{1}{1+\epsilon} +\vecind{\vv}{2}^2$.
Then, for the other task, $m=1$, it holds that,
\begin{align*}
    f(\x_{1};\Bar{\Theta}_{t})
    &=(1-\varepsilon)\vecind{\vu_t}{2}^2 -(1-\epsilon)\vecind{\vv_t}{2}^2
    \\&= -(1-\varepsilon)\vecind{\vv_t}{2}^2 + (1-\varepsilon) (-\frac{1}{1+\epsilon} +\vecind{\vv_t}{2}^2)
    \\&=-\frac{1-\epsilon}{1+\epsilon}
    \\&\leq -0.9\,,
\end{align*}
and, $F_{1}(\Bar{\Theta}_{t})=|f(\x_{1};\Bar{\Theta}_{t}) - y_1|^2\geq 3$.
\end{proof}

\newpage

\subsection{Proof of \cref{thm:local_convergence_reg}}
In this section, we analyze \cref{alg:sequential_mm} for the multi-task regression problem. Each task $m$ imposes a set of equality constraints:
\[ h_{i,m}(\Theta) = f(\x_i^{(m)}; \Theta) - y_i^{(m)} = 0. \]
Let $\fset_{m,i} = \{ \Theta \in \mathbb{R}^\paramdim \mid h_{i,m}(\Theta) = 0 \}$ be the feasible set for a single data point. The global solution set is the intersection $\interset = \bigcap_{m,i} \fset_{m,i}$.
We discuss the $\beta$-smooth and $G$-Lipschitz case.

Now we turn to prove $\kappa$-linear regularity for homogeneous models with inequality constraints.

\begin{lemma}[$\kappa$-Linear Regularity of Homogeneous Models]\linebreak
\label{lem:homo_linear_reg}
Consider a \(\homogendeg\)-positively-homogeneous model $f(\cdot;\Bar \Theta):\mathcal X\to\mathbb R$ that is $G$-Lipschitz and $\beta$-$smooth$.
Assume joint separability, i.e., nonempty intersection $\interset=\fset_{1}\cap \dots\cap\fset_{M}\neq \emptyset$.
Let $y_{\min}=\min_{m,i} |y_i^{(m)}|\neq 0$ and $\opt\in\interset$.
Let $\epsilon$ holding $\frac{1}{(1-\epsilon)^{M-1}} = 1 + \frac{\homogendeg^2}{2(M-1)G^2\|\opt\|^2}$.
Then, for $\delta=\frac{\epsilon r y_{\min}}{\beta \|\Theta^\star\|}$, the collection $\{\fset_m\}_{m=1}^M$ satisfies $\kappa$-linear-regularity on $\ball_\delta(\opt)$ with constant
$\kappa = \frac{G \|\opt\|}{\homogendeg y_{\min}}$. 
\end{lemma}

\begin{proof}
We define the constraint functions as $h_{i,m}(\Theta) = f(\x_i^{(m)}; \Theta) - y_i^{(m)} = 0$.
By Euler's Homogeneous Function Theorem,
\[
    \langle \nabla h_{i,m}(\opt), \opt \rangle = \langle \nabla_\Theta f(\x_i^{(m)}; \opt), \opt \rangle=\homogendeg f(\x_i^{(m)})(\opt)\opt= \homogendeg y_i^{(m)}
    \geq \homogendeg y_{\min}.
\]
and, if $h_m$ is the function $h_m=(h_{1,m}\ldots,h_{n_m,m})$,
\[
\|\nabla h_m(\opt)^\top \bm{\lambda}\| \ge \frac{|\langle \opt, \nabla h_m(\opt)^\top \bm{\lambda} \rangle|}{\|\opt\|} = \frac{|\sum_j \vecind{\bm{\lambda}}{j} \langle \nabla h_{j}(\opt), \opt \rangle|}{\|\opt\|} \ge \frac{\homogendeg y_{\min}\sum |\vecind{\bm{\lambda}}{j}|}{\|\opt\|}.
\]
Thus, if $\nabla h_m(\opt)^\top \bm{\lambda}=0$, this implies $\bm{\lambda}=0$, thus, metric regularity holds.
As a result, by \cref{lem:944},
for $x$ near $\opt$ it holds for 
$$\mu=\max_{\bm{\lambda}\in N_D(h_m(\opt)),\|\bm{\lambda}\|=1}\frac{1}{\nabla h_m(\opt)\bm{\lambda}}\leq\frac{\|\opt\|}{y_{\min}\homogendeg},$$
that
\[
    \dist(\Theta, \interset) \le \mu \max_{i,m} |h_{i,m}(\Theta)| 
    \leq \frac{\|\opt\|}{\homogendeg y_{\min}} \max_{i,m} |h_{i,m}(\Theta)|.
\]
Now, by local $G$-Lipschitzness, for every $i,m$ it holds that, 
\[|h_{i,m}(\Theta)|=|h_{i,m}(\Theta)-h_{i,m}(\proj_{\fset_{m}}(\Theta)|\leq G \dist (\Theta,\fset_m)\]
Since this holds also for the maximal example,
we get, 
\[
    \dist(\Theta, \interset) \leq \frac{G \|\opt\|}{\homogendeg y_{\min}} \max_{m \in M} \dist(\Theta, \fset_m).
\]
\end{proof}

For $(\epsilon,\delta)$-regularity, we have the following lemma.

\begin{lemma}[$(\varepsilon, \delta)$-Regularity of Homogeneous Models]
\label{lem:homo_eps_delta}
Consider an \(\homogendeg\)-positively-homogeneous model $f(\cdot;\Bar \Theta):\mathcal X\to\mathbb R$ that is $G$-Lipschitz and $\beta$-$smooth$.
Assume joint separability, i.e., nonempty intersection $\interset=\fset_{1}\cap \dots\cap\fset_{M}\neq \emptyset$.
Let $y_{\min}=\min_{m,i} |y_i^{(m)}|\neq 0$ and $\opt\in\interset$.
Then, for every $\varepsilon > 0$ and $\delta=\frac{\epsilon r y_{\min}}{\beta \|\Theta^\star\|}$, every feasible set $\fset_m$ is $(\varepsilon, \delta)$-regular at $\opt$.
\end{lemma}

\begin{proof}
Let $\fset_m = \{ \Theta \mid h_m(\Theta) = 0 \}$. Since $h_m$ is smooth, its gradient $\nabla h_m(\Theta)$ is continuous.

To prove $(\varepsilon, \delta)$-regularity, we need to show that for any $\vx, \y \in \fset_m \cap \ball_\delta(\opt)$ and any proximal normal vector $\vv \in N^{prox}_{\fset_m}(\opt)$, the following holds:
\[ \langle \vv, \y - \vx \rangle \le \varepsilon \|\vv\| \|\y - \vx\|. \]
Let $\vv\in N^{prox}_{\fset_m}(\opt)$ be such a vector.
By Euler's Homogeneous Function Theorem,
\begin{equation}
    \label{eq:euler}
    \langle \nabla h_{i,m}(\opt), \opt \rangle = \langle \nabla_\Theta f(\x_i^{(m)}; \opt), \opt \rangle=\homogendeg f(\x_i^{(m)})(\opt)\opt= \homogendeg y_i^{(m)}
    \geq \homogendeg y_{\min}.
\end{equation}
If $h_m$ is the function $h_m=(h_{1,m}\ldots,h_{n_m,m})$, \begin{equation}
\label{eq:lam_norm}
\|\nabla h_m(\opt)^\top \bm{\lambda}\| \geq  \frac{|\sum_j \vecind{\bm{\lambda}}{j} \langle \nabla h_{j,m}(\opt), \opt \rangle|}{\|\opt\|} \ge \frac{\homogendeg y_{\min}\sum |\vecind{\bm{\lambda}}{j}|}{\|\opt\|}  \ge \|\bm{\lambda}\|\frac{\homogendeg y_{\min}}{\|\opt\|}.
\end{equation}
Since $h_m$ is smooth vector valued mapping we have $\beta'$ such that,
\[
\|h_m(\y) - h_m(\vx) - \nabla h_m(\vx)(\y - \vx)  \|\leq \frac{\beta'}{2}\|\y - \vx\|^2
\]
By \cref{lem:rw_614}, and using the Cauchy-Schwarz inequality on the error term, it holds that:
\begin{align*}
    \langle \vv, \y - \vx \rangle 
    &= \langle \nabla h_m(\vx)^\top \bm{\lambda}, \y - \vx \rangle \\
    &= \langle \bm{\lambda}, \nabla h_m(\vx)(\y - \vx) \rangle \\
    &= \langle \bm{\lambda}, h_m(\y) - h_m(\vx) - \left( h_m(\y) - h_m(\vx) - \nabla h_m(\vx)(\y - \vx) \right) \rangle \\
    &= \langle \bm{\lambda}, h_m(\y) - h_m(\vx) \rangle - \langle \bm{\lambda}, h_m(\y) - h_m(\vx) - \nabla h_m(\vx)(\y - \vx) \rangle \\
    &\leq \langle \bm{\lambda}, h_m(\y) - h_m(\vx) \rangle + \|\bm{\lambda}\| \left\| h_m(\y) - h_m(\vx) - \nabla h_m(\vx)(\y - \vx) \right\| \\
    &= \frac{\beta'
    \|\bm{\lambda}\|}{2} \|\y - \vx\|^2.
\end{align*}
Thus, by \cref{eq:lam_norm},
to satisfy the $(\varepsilon, \delta)$-regularity condition, we denote $\kappa= \frac{\|\opt\|}{\homogendeg y_{\min}}$. We need to choose $\delta$ such that for every $\x, \y \in \ball_\delta(\bar \x)$
\[
\frac{\kappa \beta'}{2} \|\vv\| \|\y - \vx\|^2 \le \varepsilon \|\vv\| \|\y - \vx\|.
\]
Dividing by $\|\vv\| \|\y-\x\|$ (assuming nonzero, otherwise the inequality holds trivially), this requires $\frac{\kappa \beta'}{2} \|\y - \vx\| \le \varepsilon$.
Since $\x, \y \in \ball_\delta(\bar \x)$, we have $\|\y - \vx\| \le 2\delta$.
Thus, it suffices to choose $\delta$ such that:
\[
\frac{\kappa \beta'}{2} (2\delta) \le \varepsilon \implies \delta \le \frac{\varepsilon}{\kappa \beta'}.
\]
With this $\delta$, the set is $(\varepsilon, \delta)$-regular.
\end{proof}

\begin{proof}[of \cref{thm:local_convergence_reg}]
The proof follows identically to the classification case, utilizing the established $(\epsilon, \delta)$-regularity and $\kappa$-linear regularity for the equality-constrained sets.
For the forgetting, the statement follows by the fact that, for every $m$, 
\begin{align*}
    F_m(\Bar{\Theta}_t)&=
    \max_i \left|f(\x_i^{(m)};\opt_t)-y_i^{(m)}\right|^2
    =\max_i \left|f(\x_i^{(m)};\Bar{\Theta}_t)-f(\x_i^{(m)};\proj_{\interset}(\opt)\right|^2
    \\&\leq G^2\dist^2(\Bar{\Theta}_t,\interset).
\end{align*}
\end{proof}